\documentclass[review]{elsarticle}

\usepackage[margin=3cm]{geometry}
\usepackage{lscape,dsfont}
\usepackage[figuresright]{rotating}
\usepackage{graphicx,url}
\usepackage{booktabs}
\usepackage{siunitx}
\usepackage{makeidx}
\usepackage{subcaption}
\usepackage{array}
\usepackage{amssymb}
\usepackage{enumerate}
\usepackage{amsmath}
\usepackage{bigstrut}
\usepackage{amsopn}
\usepackage{algpseudocode}
\usepackage{float}
\usepackage{bm}
\usepackage{rotating}
\usepackage{multicol}
\usepackage{multirow}
\usepackage{tabularx}
\usepackage{arydshln}
\usepackage{color}
\usepackage{amsthm}
\newtheorem{theorem}{Theorem}

 \newtheorem{definition}{Definition}
\newtheorem{corollary}{Corollary}
\usepackage[linesnumbered,vlined,ruled]{algorithm2e}
\theoremstyle{definition}
\newtheorem{observation}{Observation}

\usepackage{bbm}

\usepackage{hyperref}
 \usepackage{color}
\usepackage{blindtext}
\usepackage{scrextend}
\addtokomafont{labelinglabel}{\sffamily}

\usepackage[misc,geometry]{ifsym}
\DeclareMathOperator*{\argmax}{arg\,max}
\DeclareMathOperator*{\argmin}{arg\,min}

\definecolor{DSgray}{cmyk}{0,1,0,0}

\begin{document}
   \begin{frontmatter}

\title{The Impact of Isolation Kernel\\ on Agglomerative Hierarchical Clustering Algorithms}

 \cortext[mycorrespondingauthor]{Corresponding author}

\author[poli]{Xin Han}
\ead{xin.han@tulip.org.au}
\author[deakin]{Ye Zhu \corref{mycorrespondingauthor}}
\ead{ye.zhu@ieee.org}
\author[kaiming]{Kai Ming Ting}
\ead{tingkm@nju.edu.cn}
\author[deakinSRC]{Gang Li}
\ead{gang.li@deakin.edu.au}
\address[poli]{School of Computer Science, Xi'an Shiyou University, Shaanxi, China 710065}
\address[deakin]{School of Information Technology, Deakin University, Geelong, Australia}
\address[deakinSRC]{Centre for Cyber Security Research and Innovation, Deakin University, Geelong, Australia}
\address[kaiming]{School of Artificial Intelligence, Nanjing University, Nanjing, China 210023}

\begin{abstract}
Agglomerative hierarchical clustering (AHC) is one of the popular clustering approaches. Existing AHC methods, which are based on a distance measure, have one key issue: it has difficulty in identifying adjacent clusters with varied densities, regardless of the cluster extraction methods applied on the resultant dendrogram. In this paper, we identify the root cause of this issue and show that the use of a data-dependent kernel (instead of distance or existing kernel) provides an effective means to address it. We analyse the condition under which existing AHC methods fail to extract clusters effectively; and the reason why the data-dependent kernel is an effective remedy.  This leads to a new approach to kernerlise existing hierarchical clustering algorithms such as  existing traditional AHC algorithms, HDBSCAN, GDL and PHA. In each of these algorithms, our empirical evaluation shows that a recently introduced Isolation Kernel produces a higher quality or purer dendrogram  than distance, Gaussian Kernel and adaptive Gaussian Kernel. 
\end{abstract}

\begin{keyword}
Agglomerative hierarchical clustering, Varied densities, Dendrogram purity, Isolation kernel, Gaussian kernel.
\end{keyword}

\end{frontmatter}
\section{Introduction}
 
Hierarchical clustering is one of the widely used clustering methods~\citep{jainDataClusteringReview1999,gilpin2013formalizing,cohen2019hierarchical}.
Given a set of data points,
the goal of hierarchical clustering is not to find a single partitioning of the data, but a hierarchy of subclusters in a dendrogram. Because its clustering output of a dendrogram is easy to interpret,
hierarchical clustering has been  used in a wide range of applications,
e.g., social networks analysis~\citep{rajaramanMiningMassiveDatasets2011},
bioinformatics~\citep{diezpalacioNovelBrainPartition2015},
text classification~\citep{Malik2010,Zhao2005}
and financial market analysis~\citep{tumminelloCorrelationHierarchiesNetworks2010}.

The most widely used hierarchical clustering
approach is agglomerative hierarchical clustering (AHC)~\citep{jainDataClusteringReview1999}.
It  starts with subclusters of individual points,
and then iteratively merges two most similar subclusters,
based on a \textit{linkage function} which measures the similarity of two subclusters,
until all points belong to a single cluster.
AHC has been studied in the theoretical community and
used by practitioners~\citep{wu2009towards,heller2005bayesian,Dasgupta2016CostFunction,cohen2019hierarchical}.
 
The linkage function used in an AHC relies on a distance (or similarity) measure.
Many distanced-based linkage functions have been proposed for hierarchical clustering, such as complete linkage~\citep{kingStepwiseClusteringProcedures1967} and average linkage~\citep{sokal1958u}. In addition, in order to capture complex structures in the data, the graph-structural agglomerative clustering algorithms such as GDL~\citep{zhang2012graph} and PIC~\citep{zhang2013agglomerative} have used a linkage function based on a $k$-nearest-neighbour graph. 

This research is motivated by the current state of two partitioning clustering research fronts. First,
the impact of varied densities of clusters on density-based partitioning clustering algorithms has been well studied (e.g. \cite{ankerst1999optics,ertoz2003finding,campello2013density,zhu2016density,ting2019lowest,qin2019nearest}). But its impact on
the traditional AHC algorithms (T-AHC)
has not been investigated in the literature thus far.
We think its impact has been overlooked because the dendrogram is said to have a `complete' set of subclusters, assuming that some of these subclusters will be a good match for the true clusters. Yet, we show that this `complete' set of subclusters often does not include ones which are a good match to the ground truth clusters in a given dataset.

Our investigation uncovers a specific bias of T-AHC: using the distance-based linkage function, T-AHC tends to merge high-density subclusters first, before low-density subclusters in the merging process.  While there is a hint of this bias in the literature \cite{klemela2009smoothing}, no analysis of its cause has been conducted, as far as we know.  

Another possible reason for its lack of attention in this matter is that this bias only becomes an issue if clusters are not well separated (see our definition in Section \ref{why}.) In many real-world datasets in which clusters are not well separated, we have observed that this bias often leads to a dendrogram of poor quality.

%\textcolor{red}{Need to describe how they are related to a linkage function. Can we call them graph-based linkage function?}.
%However, these hierarchical clustering algorithms have difficulty in handling clusters with varied densities because their linkage functions are typically based on a distance measure.

%There are many methods which can extract clusters from different levels of a dendrogram by having cuts of different levels. For example, HDBSCAN \cite{campello2013density} and FOSC-OPTICSDend \cite{campello2013framework} can retrieval all possible density-based clusters subsumed in the dendrogram by enumerating all cuts on the dendrogram and use  an optimisation method to output the most important clusters. However, these algorithms still produce poor clustering performance due to the dendrogram quality influenced by the bias of the distance-based linkage function. 

The impact of the bias on T-AHC is most revealing in a dataset, having clusters of varied densities and not well separated, as shown  in Figure \ref{fig:hard2}.  Clusters extracted from this dendrogram either lost many of their (true) members or they have (false) members of other clusters or both.

\begin{figure} [!htb]
     \centering
     \begin{subfigure}[b]{0.43\textwidth}
         \centering
         \includegraphics[width=\textwidth]{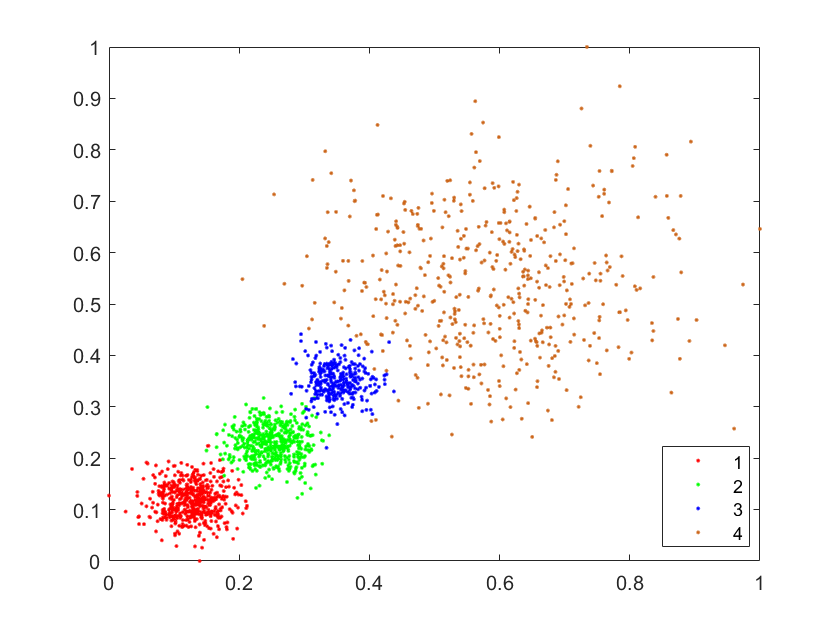}
           \caption{Dataset: 4 clusters of different densities}
         \label{fig:hard21}
     \end{subfigure}
     \begin{subfigure}[b]{0.43\textwidth}
         \centering
       \includegraphics[width=\textwidth]{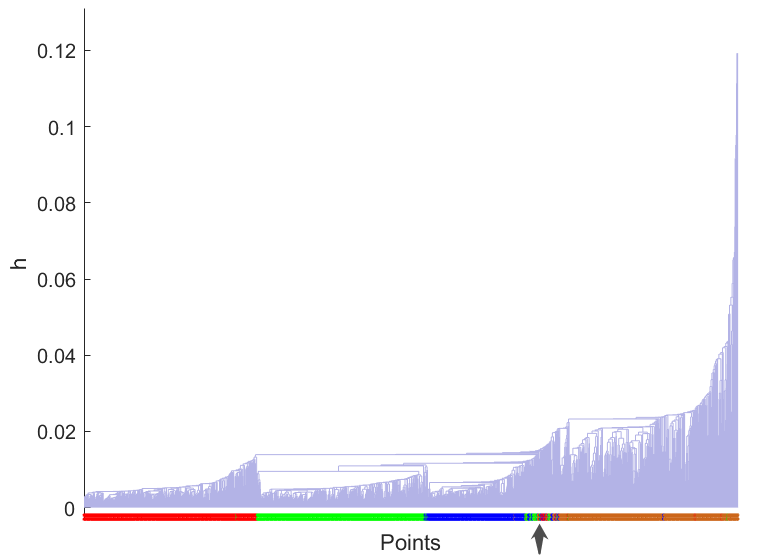}
         \caption{Dendrogram produced by T-AHC} 
         %$Dendrogram \ purity=0.78$}
       \label{fig:hard122}
     \end{subfigure}
    \caption{A dendrogram produced by T-AHC with the distance-based single-linkage function on a dataset with four clusters of varied densities.  The colours at the bottom of the dendrogram correspond to the true cluster labels of all points shown in Figure (a). The arrow in Figure (b) indicates the subtrees containing points from different clusters.}
    \label{fig:hard2}
\end{figure}

Second, in the context of density-based partitioning clustering, the root cause of the bias has been established, i.e., the use of data-independent kernel/distance \cite{ertoz2003finding,ting2019lowest,qin2019nearest}. In addition, in the context of kernel/spectral (partitioning) clustering,  the density bias has been recognised to be an issue \cite{huang2015density,8166757}. Stop short of declaring that the use of data-independent kernel is the root cause, a kernel which adapts to local density to replace a data-independent kernel has been proposed as a remedy,
%For example, spectral clustering has a local density bias in the normalised cut process such that boundary points across two neighbouring clusters with different densities are likely to be linked together \cite{huang2015density}. Kernel k-means clustering has a density bias such that small dense group are more likely to be separated than the rest \cite{8166757}. Some data-dependent approaches have been used to reduce the density bias in kernel clustering based on either locally adaptive weights or locally adaptive kernels, 
e.g., Adaptive Gaussian Kernel \cite{zelnik2005self}. To the best of our knowledge, using the data-dependent kernel function to reduce the density bias in hierarchical clustering is still unexplored. %However, the study on the impact of kernel on hierarchical clustering remains unexplored in the literature.

Here, we contend that the same root cause yields the bias in T-AHC. We then provide the formal analysis and explanation as to why a well-defined data-dependent kernel called Isolation Kernel \cite{ting2018IsolationKernel,qin2019nearest} is an effective remedy.

%In this paper, we investigate both data-independent and data-dependent kernels in improving the clustering performance of hierarchical clustering algorithms in dealing with datasets having hugely different cluster densities.
The contributions of this paper are summarised as follows:

\begin{enumerate}
%    \item Introducing new definitions of kernel-based clusters for agglomerative hierarchical clustering;
    \item Providing the formal condition under which an AHC, that employs an existing kernel/distance, does not extract clusters of a dataset effectively.
    \item Introducing a new concept called \textbf{entanglement} as a way to explain the merging process that leads to a poor quality dendrogram. Two indicators, i.e., the number of entanglements and the average entanglement level, are shown to be highly correlated to an objective measure of goodness of dendrogram called dendrogram purity. 
    \item Identifying the root cause of a density bias of T-AHC. The bias merges points in dense region first, before merging points in the sparse region; and its root cause is due to the use of data-independent distance/kernel. Though this analysis is mainly based on T-AHC, the same root cause also applies to existing AHC algorithms as well.  
    \item Proposing a generic approach to improve existing distance-based AHC algorithms: simply replace the distance function with a data-dependent kernel, without modifying the algorithms. We also provide the reason why a data-dependent kernel can significantly reduce the above-mentioned bias;

    \item Presenting the empirical evaluation results using four algorithms, i.e., T-AHC~\citep{jainDataClusteringReview1999}, HDBSCAN~\citep{campello2013density},
    GDL~\citep{zhang2012graph} and PHA~\citep{lu2013pha} that: 
    
    (i)  kernels produce better clustering results than distance, and
    %when there are varied density clusters closely located, and 
    
    (ii) a recently introduced Isolation Kernel \citep{ting2018IsolationKernel,qin2019nearest}
     performs better than Gaussian Kernel and Adaptive Gaussian Kernel \citep{zelnik2005self}.
     % especially in dealing with datasets having hugely different cluster densities.
\end{enumerate}

Our approach is distinguished from those used in existing AHC algorithms in two ways:
\begin{itemize}
    \item Rather than creating new linkage functions that still employ distance measure \citep{ackerman2016characterization,pmlr19}, we propose to use a data-dependent kernel to replace the distance function in existing linkage functions.
    \item The methodology is generic which can be applied to different hierarchical clustering algorithms.
    Many existing linkage functions are tailored made for a specific algorithm.
    We show that our approach can be applied to four existing methods, T-AHC, HDBSCAN \citep{campello2013density}, GDL~\citep{zhang2012graph} and PHA~\citep{lu2013pha},
    even though each algorithm has its own specific linkage function(s).
\end{itemize}

The rest of the paper is organised as follows.
We first describe the related work on AHC and kernels in Section 2. Then we define the  Kernel-based AHC and the condition under which it fails to identify clusters in Section 3. Section 4 investigates the effectiveness of the use of a data-dependent kernel in addressing the density bias in Kernel-based AHC. An extensive empirical evaluation is reported in Section 5, followed by the conclusions in the last section.
%\textcolor{red}{This paragraph needs to be rephrased.}

\section{Related Work}~\label{sec:relatedwork}

\subsection{Agglomerative Hierarchical Clustering}

Hierarchical clustering can be categorised into agglomerative (bottom-up) and divisive (top-down) methods~\citep{murtagh1983survey}, depending on the direction in which the hierarchy in a dendrogram is created.
Many research focuses on improving the hierarchical clustering
on algorithm-level~\citep{heller2005bayesian,krishnamurthy2012efficient}
and understanding hierarchical clustering~\citep{wu2009towards,Dasgupta2016CostFunction,cohen2019hierarchical}.
 
In this paper,
we focus on AHC.
An agglomerative method needs two functions:
a distance function measures the dissimilarity of two points;
and the linkage function $h(C_i,C_j)$ measures the dissimilarity of subclusters $C_i$ and $C_j$~\citep{pmlr19}. Treating individual points in a given dataset as subclusters, it iteratively merges two most similar subclusters based on $h(\cdot,\cdot)$
at each step  to form a tree-based structure (dendrogram) until a single cluster is formed at the root of the tree.

When using a hierarchical clustering algorithm, it is important to  choose a proper linkage function for the dendrogram construction because different linkage functions have different properties \citep{aggarwal2013data}. For example, the complete-linkage function is sensitive to noise and outliers; and the average linkage function finds mainly globular clusters only \citep{aggarwal2013data}. We found that four commonly used linkage functions in T-AHC have difficulty in handling clusters with varied densities (see Section \ref{why}).
 
To improve the performance of T-AHC, many variants of these common linkage functions have been introduced, attempting to address their limitations.
HDBSCAN~\citep{campello2013density} uses a density-based linkage function to
find clusters of varied densities.\footnote{HDBSCAN can be interpreted as a kind of AHC algorithm
which relies on a single linkage function and a particular dissimilarity measure, see the details in~\ref{appendA}.}
PHA~\citep{lu2013pha}, a potential-based hierarchical agglomerative clustering method based on a potential
theory~\citep{shuming2002potential}, uses a potential field produced by the distance of all the data points to measure the similarity between clusters.
The potential field is interpreted to represent the global data distribution information.
PHA is claimed to be robust to different types of data distribution in a dataset.

In order to capture the complex structure of a dataset,
several graph-structural agglomerative clustering algorithms are proposed.
Those algorithms first create a $k$-nearest-neighbour  graph to obtain a set of small initial clusters.
Then it iteratively merges two most similar clusters until the target number of clusters is obtained. 
Chameleon~\citep{karypis1999chameleon} measures the similarity
between two clusters based on relative interconnectivity and relative closeness, both of which are defined on the graph.
Zell~\citep{zhao2009cyclizing} describes the structure of a cluster and defines the similarity between two clusters based on the structural changes after merging.
GDL~\citep{zhang2012graph} uses the product of average indegree and average outdegree in graphs
to measure the similarity between two clusters. These algorithms typically use the pairwise distance to build the neighbourhood graph.
Although they could handle clusters with varied densities to some degree,
we show that simply replacing the distance with Isolation Kernel is able to significantly improve their clustering performance.   

All the above hierarchical clustering algorithms are based on a distance function to construct their linkage functions. As existing kernels like Gaussian kernel are data-independent, just as the distance function, the baseline of our investigation is AHC using kernel-based linkage functions, where the kernel is the commonly used Gaussian kernel. We will see later that both distance-based and this kernel-based linkage functions lead to the same bias we have mentioned in the introduction.

\subsection{Kernels}

Various kernel-based clustering algorithms have been developed to improve the performance of existing distance-based machine learning and data mining algorithms,
including kernel k-means~\citep{scholkopf1998nonlinear,shawe2004kernel},
density-based clustering~\citep{hinneburg2007denclue,qin2019nearest},
spectral clustering~\citep{dhillon2004kernel,zelnik2005self,kang2018unified},
NMF~\citep{xu2003document},
kernel SOM~\citep{macdonald2000kernel}
and kernel neural gas~\citep{qin2004kernel}.
However,
the study on the impact of a
kernel on hierarchical clustering remains unexplored in the literature.

In this paper, we focus on a commonly used data-independent kernel, i.e., Gaussian Kernel, and two data-dependent kernels, i.e., Adaptive Gaussian Kernel \citep{zelnik2005self} and Isolation Kernel \citep{ting2018IsolationKernel,qin2019nearest}; and examine their impacts on AHC. The former has been applied to spectral clustering, and the latter has been applied to SVM classifiers and DBSCAN \cite{ester1996density}.

A brief description of each of these kernels is provided in the following.
For any two points ${ x}, { y} \in \mathbb{R}^d$,
\begin{description}
\item[\textbf{Gaussian Kernel}:]
%Given a dataset $D=\{x_1, \dots, x_n \}$ in $\mathbb{R}^d$.
The Gaussian Kernel defines the similarity between $x$ and $y$ as follows:
\[
\mathcal{K}(x,y) = exp(\frac{-\parallel x - y \parallel^2}{2\sigma^2})
\]
\noindent where $\sigma$ is the bandwidth of the kernel.

Gaussian Kernel is a commonly used kernel, e.g., SVM for classification \citep{scholkopf2002learning} and  t-SNE \citep{maaten2008visualizing} for visualisation.

\item[\textbf{Adaptive Gaussian Kernel}:]
In order to make the similarity adaptive to local density, Adaptive Gaussian Kernel \citep{zelnik2005self} is defined as:

  \begin{eqnarray}
K_{AG}({ x},{ y})  & = &   exp(\frac{-|| x- y||^2}{\sigma_{ x} \sigma_{ y}})
 \label{Eqn_AG}
\end{eqnarray}
%\textcolor{red}{$|| x- y||_2$? should it be $|| x- y||^2$?}
\noindent where $\sigma_{ x}$ is the Euclidean distance between $ x$ and $ x$'s $k$-th nearest neighbour.

Adaptive Gaussian Kernel was introduced in spectral clustering to
adjust the similarity locally to perform dimensionality reduction before clustering \citep{zelnik2005self}.

\item[\textbf{Isolation Kernel}:]
Isolation Kernel is a recently introduced data-dependent kernel, which adapts to local
distribution. The pertinent details of Isolation Kernel \citep{ting2018IsolationKernel,qin2019nearest} are provided below.

Let $\mathds{H}_\psi(D)$ denote the set of all partitions $H$
that are admissible under the dataset $D$, where each $H$ covers the entire space of $\mathbb{R}^d$; and each of the $\psi$ isolating partitions, $\theta[{z}] \in H$, isolates one data point ${z}$ from the rest of the points in a random subset $\mathcal D \subset D$, and $|\mathcal D|=\psi$. Here we use the Voronoi diagram \cite{aurenhammer1991voronoi} to partition the space, i.e., each $H \in \mathds{H}_\psi(D)$ is a Voronoi diagram and each sample point $z\in \mathcal D$ is a cell centre.

\begin{definition} Isolation Kernel of ${ x}$ and ${ y}$ wrt $D$ is defined to be
	the expectation taken over the probability distribution on all partitions $H \in \mathds{H}_\psi(D)$ that both ${ x}$ and ${ y}$  fall into the same isolating partition $\theta[{ z}] \in H, { z} \in \mathcal{D}$:
	\begin{eqnarray}
K_\psi({ x},{ y}\ |\ D) &=&  {\mathbb E}_{\mathds{H}_\psi(D)} [\mathds{1}({ x},{ y} \in \theta[{ z}]\ | \ \theta[{ z}] \in H)] \nonumber \\
&=& {\mathbb E}_{\mathcal{D} \subset D} [\mathds{1}({ x},{ y}\in \theta[{ z}]\ | \ { z}\in \mathcal{D})]  \nonumber
\\
&=& P({ x},{ y}\in \theta[{ z}]\ | \ { z}\in \mathcal{D} \subset D)
		\label{eqn_kernel}
	\end{eqnarray}
where $\mathds{1}(\cdot)$ is an indicator function.
\end{definition}

In practice, Isolation Kernel $K_\psi$ is constructed using a finite number of partitions $H_i, i=1,\dots,t$, where each $H_i$ is created using $\mathcal{D}_i \subset D$:
\begin{eqnarray}
K_\psi({ x},{ y}\ |\ D)  & = &  \frac{1}{t} \sum_{i=1}^t   \mathds{1}({ x},{ y} \in \theta\ | \ \theta \in H_i) \nonumber\\
 & = & \frac{1}{t} \sum_{i=1}^t \sum_{\theta \in H_i}   \mathds{1}({ x}\in \theta)\mathds{1}({ y}\in \theta)
 \label{Eqn_IK}
\end{eqnarray}
\noindent where $\theta$ is a shorthand for $\theta[{ z}]$, and $\psi$ is the sharpness parameter\footnote{This parameter corresponds to the $\sigma$ parameter in the Gaussian kernel, i.e., the smaller $\sigma$ is, the more concentrated its kernel distribution.}, i.e., the larger $\psi$ is, the sharper its kernel distribution.

As Equation (\ref{Eqn_IK}) is quadratic, $K_\psi$ is  a valid kernel.
The larger the $\psi$, the sharper the kernel distribution. $\psi$ is a parameter having a similar function to $\sigma$ in the Gaussian Kernel, i.e., the smaller $\sigma$ is, the narrower the kernel distribution. 
\end{description}

\subsection{Dendrogram evaluation}
An AHC algorithm produces a dendrogram or cluster tree. 
To evaluate the goodness of a dendrogram, \textit{Dendrogram Purity}  has been created to measure the hierarchical clustering result \citep{heller2005bayesian,kobren2017hierarchical,monath2019scalable}. %\textcolor{red}{This sentence is ambiguous. [10] created it but never developed it??}
Given a dendrogram, the procedure finds the smallest subtree containing two leave nodes belonging to the same ground-truth cluster; and measures the fraction of leave nodes in that subtree which belongs to the same cluster. The dendrogram purity is the expected value of this fraction. It is 1 if and only if all leave nodes belonging to the same cluster are rooted in the same subtree. 
%As a result, each cluster can be extracted by a cut on the dendrogram. 
The dendrogram purity \citep{heller2005bayesian} is calculated as follows.

Given a dendrogram $\mathcal{T}$ produced from a dataset $D=\{x_1,\dots,x_n\}$.
Let $C^{\star} = \left\{ C_{t}^{\star}\right\}_{t=1}^{\kappa}$
be the true labels in $\kappa$ clusters,
and $\mathcal{P}^{\star}=\left\{\left(x, x^{\prime}\right) | x, x^{\prime} \in \mathcal{X}, C^{\star}(x)=C^{\star}\left(x^{\prime}\right)\right\}$ be
the set of pairs of points that
are in the same ground-truth cluster.
Then the dendrogram
purity of $\mathcal{T}$ is
\begin{equation}~ 
\mathrm{Purity}(\mathcal{T})=\frac{1}{\left|\mathcal{P}^{\star}\right|} \sum_{t=1}^{\kappa} \sum_{x_{i}, x_{j} \in C_{t}^{\star}} \operatorname{pur}\left(\operatorname{lvs}\left(\operatorname{LCA}\left(x_{i}, x_{j}\right)\right), C_{t}^{\star}\right)
\label{eqn_dpurity}
\end{equation}
where $\operatorname{LCA}\left(x_{i}, x_{j}\right)$ is the least common ancestor of $x_i$ and $x_j$ in $\mathcal{T}$; $\operatorname{lvs}(z)\subset D$  is this the set of points in all the descendant leaves of $z$ in $\mathcal{T}$; and $\operatorname{pur}(S, C_{t}^{\star})=\frac{|S \cap C_{t}^{\star}|}{S}$  computes the fraction of $S$ that matches the ground-truth label of $C_{t}^{\star}$.

We use dendrogram purity as an objective measure to assess the quality of the dendrograms produced by different clustering algorithms.

\section{Kernel-based hierarchical clustering}

An agglomerative method creates a dendrogram and extracts clusters from it. To build a dendrogram, it iteratively merges two most similar subclusters (as measured by a linkage function based on a similarity measure) until all points in a dataset are grouped into one cluster.

A dendrogram contains a complete set of all possible subclusters grouped by the linkage function. It is richer than a flat clustering result produced by a partitioning clustering algorithm; and it shows the hierarchical relationship between subclusters. 

To extract the most meaningful clusters from a dendrogram, a cluster extraction algorithm applies cuts in the dendrogram to produce subtrees, where each subtree represents a cluster. 

Many kernel methods have been proposed to improve the performance of existing distance-based machine learning and data mining algorithms. For clustering, the use of kernel enables a method to capture the nonlinear relationship inherent in the data distribution and to separate non-convex clusters that would otherwise be impossible for distance-based methods. For example, kernel $k$-means \citep{scholkopf1998nonlinear} and spectral clustering \citep{zelnik2005self} have been shown to enrich the types of clusters that can be detected by distance-based k-means.

%Since a pairwise distance matrix is required as the input, as kernel $k$-means do,
An existing distance-based agglomerative clustering algorithm can be easily kernelised by replacing the distance matrix with the kernel similarity matrix, leaving the rest of the procedure unchanged.  The procedure is shown in Algorithm~\ref{alg:khc} which employs $\hslash$ as a kernel linkage function.\footnote{When using a dissimilarity linkage function such as a Euclidean distance function, the Equation under the line 6 in Algorithm~\ref{alg:khc} should be $\argmin_{C_{\imath}\ne C_{\jmath} \in \mathbbm{C}} h(C_{\imath},C_{\jmath})$.} By setting $\kappa$ to 1, it produces a dendrogram. Table \ref{Kernel} shows the kernel versions of the four commonly used linkage functions in T-AHC. 

\begin{table*}[!htb]
  \renewcommand{\arraystretch}{1}
   \setlength{\tabcolsep}{0pt}
  \centering
  \caption{Kernel-based linkage functions $\hslash$ used in T-AHC. $C$ is a cluster which consists of data points; and $K$ is a kernel function.}
    \begin{tabular}{cc}
    \toprule
%    Linkage  & kernel-based \\     \midrule
    Single-linkage    & $\hslash(C_i,C_j)= \max_{x\in C_i y\in C_j} K(x,y) $ \\
    Complete-linkage      &  $\hslash(C_i,C_j)= \min_{x\in C_i y\in C_j} K(x,y) $ \\
    Average-linkage  &   $\hslash(C_i,C_j)= \frac{1}{|C_i||C_j|} \sum_{x\in C_i y\in C_j} K(x,y) $  \\
    Weighted-linkage  &  \begin{minipage}{8cm}\begin{equation}
    \hslash(C_i,C_j) = \left\{ \begin{array}{l}
       \frac{ \hslash(C_p,C_j)+\hslash(C_q,C_j)}{2},  \text{ if } C_i=C_q \cup C_p\\
     K(C_i, C_j), \text{ if } |C_i|=|C_j|=1
    \end{array} \right. \nonumber
\end{equation} \end{minipage} \\
        \bottomrule
    \end{tabular}%
  \label{Kernel}%
\end{table*}%

\begin{algorithm}[!!htb]
\SetAlgoLined
\SetKwInOut{Input}{Input}\SetKwInOut{Output}{Output}
\Input{$M$ - pairwise similarity matrix ($n \times n$ matrix); $\kappa$ - target number of clusters}
\Output{$\mathbbm{C} = \{C_1,C_2,\dots,C_\kappa\}$}
\BlankLine
\For{ $j = 1,2,\dots,n$ \KwTo }
    { $C_j=\{x_j\}$;}
$\mathbbm{C}=\{C_1, C_2, ..., C_n\}$ \;
\While{$|\mathbbm{C}|>\kappa$ \KwTo}
     {Find the most similar pair $C_{p}$ and $C_{q}$ based on linkage function $\hslash$: $\argmax_{C_{\imath}\ne C_{\jmath} \in \mathbbm{C}} \hslash(C_{\imath},C_{\jmath})$\;
%  Add $C = C_{i} \cup C_{j}$ to $\mathbbm{C}$\;
%       Remove $C_{i}$ and $C_{j}$ from $\mathbbm{C}$\;
Merge $C_{p}$ and $C_{q}$ in $\mathbbm{C}$\;
%       Re-index clusters in $\mathbbm{C}$\;
    }
 \caption{ AHC - Agglomerative hierarchical clustering} %\textcolor{red}{Notations are not consistent with those used in the Related Work section.}}
 \label{alg:khc}
\end{algorithm}

\newpage
Here we provide the definitions associated with Algorithm \ref{alg:khc} and its resultant dendrogram; and a theorem stating the condition under which two ground-truth clusters can be successfully extracted from the dendrogram.

\begin{definition}  Given a set of points $D=\{x_1,\dots,x_n\}$, a set of subclusters are initialised as $\mathbbm{C}_1 = \{C_1,C_2,\dots,C_n\}$, where $C_i=\{x_i\}$. The
kernel agglomerative clustering algorithm recursively merges the two most similar subclusters $C_p$ and $C_q$ at each step $s$ and updates the set of subclusters to $\mathbbm{C}_s=\{\mathbbm{C}'_{s-1}  \cup C_{pq} \}$, where $\mathbbm{C}'_{s-1} = \mathbbm{C}_{s-1}  \setminus \{ C_p, C_q\}$ and $C_{pq}=C_p \cup C_q$, as measured by a kernel-based linkage function $\hslash$, i.e.,  $\{C_p,C_q\}=\argmax_{C_{i}\ne C_{j} \in \mathbbm{C}_{S-1}} \hslash(C_{i},C_{j})$, until all subclusters are merged into one final cluster.
\label{kerAHC}
\end{definition}

\begin{definition}
A dendrogram is a tree structure representing the order of the subcluster merging process. The height at the root of a subtree indicates the value of the kernel-based linkage function which is used to merge the two subclusters to form the subtree.
\end{definition}

\begin{definition}
A cluster extracted by a cut $\eta$ on the dendrogram is a subtree from which its next merged height is more than $\eta$.
\end{definition}

\begin{theorem}

Given two non-overlapping ground-truth clusters $\zeta_i$ and $\zeta_j$ in a dataset, to correctly identify them from the dendrogram produced by the agglomerative clustering algorithm with a kernel linkage function $\hslash$, both clusters must satisfy the following condition:

\begin{equation}
\forall_{\imath \leq I,\jmath \leq J}\ \min(\hat{\hslash}(\mathbbm{C}_\imath^i),\hat{\hslash}(\mathbbm{C}_\jmath^j)) \ge \hat{\hslash}(\mathbbm{C}_\imath^i,\mathbbm{C}_\jmath^j)
\label{eqn:condition}
\end{equation}
where $\mathbbm{C}_\imath^i$ is the set of subclusters at step $\imath$ of the process in merging members in $\zeta_i$; so as $\mathbbm{C}_\jmath^j$ in $\zeta_j$ (as defined in Definition 3);
$\hat{\hslash}(\mathbbm{C}_\imath^i)=max_{C_{p} \neq C_{q} \in \mathbbm{C}_\imath^i} \hslash(C_{p},C_{q})$ and $\hat{\hslash}(\mathbbm{C}_\imath^i,\mathbbm{C}_\imath^j)=max_{C_{p}\in \mathbbm{C}_\imath^i, C_{q} \in \mathbbm{C}_\imath^j} \hslash(C_{p},C_{q})$; and  $I$ and $J$ are the maximum numbers of steps required to merge all points in $\zeta_i$ and $\zeta_j$, respectively.
\label{Theorem1}
\end{theorem}

\begin{proof}
Given two clusters $\zeta_i$ and $\zeta_j$, an violation of Equation \ref{eqn:condition} means $\exists_{s \leq I, t \leq J}\ \min(\hat{\hslash}(\mathbbm{C}_\imath^s),\hat{\hslash}(\mathbbm{C}_\jmath^t)) < \hat{\hslash}(\mathbbm{C}_\imath^s,\mathbbm{C}_\jmath^t)$, i.e.,  $\exists_{C_p\in \mathbbm{C}_\imath^s,C_q\in \mathbbm{C}_\imath^t} \forall_{C_a\neq C_b \neq C_q \neq C_p \in \{ \mathbbm{C}_\imath^s  \cup \mathbbm{C}_\imath^t \}} \hslash(C_{p},C_{q})\ge \hslash(C_{a},C_{b})$. Using the a linkage function, subclusters $C_p$ and $C_q$ from $\zeta_i$ and $\zeta_j$ will be merged before each cluster merges its all own subclusters. Thus, the clusters which can be extracted from the dendrogram are subtrees containing points from two clusters or partial points from one cluster. 
\end{proof}

Equation \ref{eqn:condition} stipulates the condition that the linkage function shall enable each subtree to merge all members of the same cluster first, before merging with the subtree of the other cluster to form the final tree of the dendrogram. Otherwise, an entanglement has occurred.  

\begin{corollary}
An entanglement between two clusters $\zeta_i$ and $\zeta_j$ is said to have occurred in the dendrogram, produced by the agglomerative clustering algorithm with a kernel linkage function $\hslash$, when there is a violation of Equation \ref{eqn:condition} at $\imath=s,\jmath=t$ such that
\begin{equation}
\exists_{s \leq I, t \leq J}\ \min(\hat{\hslash}(\mathbbm{C}^i_s),\hat{\hslash}(\mathbbm{C}^j_t)) < \hat{\hslash}(\mathbbm{C}^i_s,\mathbbm{C}^j_t)
\label{eqn_entanglement}
\end{equation}
\end{corollary}

If one or more entanglements have occurred, then the set of all possible subclusters in the dendrogram does not contain clusters $\zeta_i$ and $\zeta_j$, but their corrupted or partial versions.

Thus, if entanglements could not be avoided (e.g., in a dataset where clusters are very close to each other or even overlapping), an hierarchical clustering algorithm shall seek to achieve (i) an entanglement at $\imath,\jmath$ as high values as possible; and (ii) the least  number of entanglements,  in order to reduce the impact of incorrect membership assignment.

We use the following two indicators to assess the severity of the entanglement in a dendrogram.

The {\em number of entanglements} in a dendrogram can be counted as follows:
At every merge, the node is labelled with either a cluster label (if the two subclusters before merging have the same cluster label% or one of them has neutral label
) or neutral (if the two subclusters have different labels or one of them has neutral label). Every time a neutral label is used in a node, an entanglement has occurred; and the number of entanglements is incremented by 1.

The entanglement level is the sum of $s$ and $t$ (in Equation \ref{eqn_entanglement}) when an entanglement occurs.
The {\em average level of entanglements} is the ratio of the total level of all entanglements and the number of entanglements, where the level of entanglement at a merge is the number of steps used to reach that merge in the process.

\section{Why using data-dependent kernel?}
\label{why}

Here we discuss the density bias of T-AHC when a data-independent kernel is used and the reason why a data-dependent kernel can deal with clusters of varied densities better than data-independent kernel in the following two subsections.

\subsection{T-AHC has density bias when a data-independent kernel/distance is used}
 
Our investigation leads to the following observation:

\begin{observation}
T-AHC has a bias which links points in a dense region ahead of points in a sparse region in general. This bias is due to the use of a data-independent kernel/distance.
\end{observation}

Single-linkage clustering builds the dendrogram where the heights is related to the 1-nearest neighbour density estimate, i.e., the points with higher heights (are merged later on the dendrogram) tend to be in low-density regions \cite{klemela2009smoothing}. %, thus it tends to produce clusters with only outliers as these low density points are easier to be extracted on the dendrogram by a global cutThis bias 
This linking bias often lead to poor clustering outcomes when adjacent  clusters  have  different densities, i.e., the boundary points from a sparse cluster may link to its neighbouring dense cluster before linking back to the sparse cluster. Note that this bias has no issue if the clusters are clearly separated.

Although other linkage function could be less sensitive to the density distribution and may eliminate the issue in the cluster boundary regions, we have the following observation:  

\begin{definition}
Two clusters $C_a$ and $C_b$ are said to be well separated if any points in the valley $V$ between these clusters have the density less than any points in either cluster, i.e., $\rho'(z) < \rho(x)$, 

\noindent
where $\rho$ is the density of individual distribution of either $C_a$ or $C_b$; $\rho'$ is the density of the joint distribution of $C_a \cup C_b$; $z \in V$ which is part of the joint distribution of $C_a \cup C_b$; and $x \in C$ of individual cluster $C_a$ or $C_b$.
\end{definition}

\begin{observation}
With a data-independent kernel-based linkage function, the bias of T-AHC creates
entanglements if the clusters are 
not well separated.
%close to each other or overlapping.
\end{observation}
%even though the dataset has clusters of the same densities. 

For example, the single-linkage function relies on the nearest neighbour similarity/distance calculation. Given two adjacent clusters $C_d$ and $C_s$, if $\exists_{x\in C_d, y\in C_s} \forall_{z\neq y \in C_s} K(x,y)>K(y,z)$, then $y$ will engage in an entanglement using this linkage function. In other words, every entanglement involves a boundary point from $C_s$ which is more similar to a boundary point from $C_d$ than other points from $C_s$.

\begin{observation}
Different linkage functions produce different degrees of entanglements but the bias remains: T-AHC links points in a dense region ahead of points in a sparse region.
\end{observation}

This can be seen from the example in Table 2, where four commonly used linkage functions shown in Table \ref{Kernel} are examined. When Gaussian kernel is used, independent of the linkage functions used, each dendrogram has low linkage $\hslash$ values in dense regions (three dense clusters) and high $\hslash$ values in sparse regions (one sparse cluster (brown)).

We contend that the root cause of this bias is the use of data-independent similarity/distance.\\
To reduce/eliminate this bias, we need a similarity which is data-dependent.

%The use of IK eliminates the AHC bias. As a result, for the same linkage function, AHC using IK is expected to produce fewer entanglements than that using Gaussian kernel. 
\begin{table} [!htbp]
\scriptsize
\centering
  \begin{tabular}{cccc}
    \hline
     &  Gaussian Kernel   & Isolation Kernel  \\
      \hline
 %     \hdashline
  & Dendrogram purity=0.77 &  Dendrogram purity=0.97 \\
       \begin{turn}{90}  \qquad \quad Single-linkage  \end{turn}&       
       \includegraphics[width=1.8in]{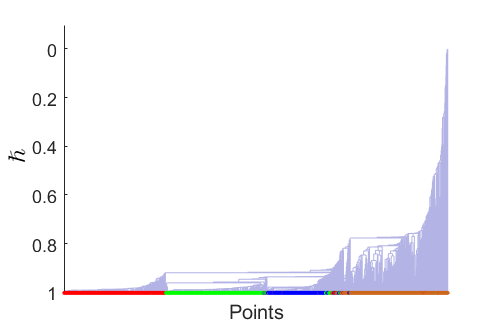}                     &
      \includegraphics[width=1.8in]{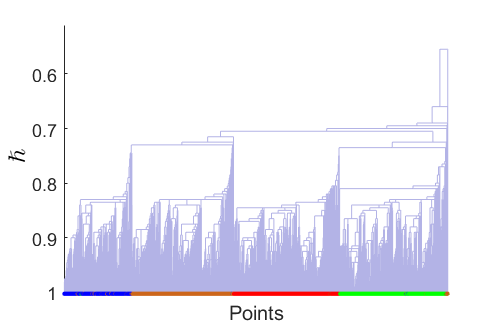}  \\
      & No.  entanglements= 242  &  No.  entanglements=170  \\
      & Avg.  entanglement level=1411 &Avg.  entanglement level=1477 \\
    \hdashline
     &Dendrogram purity=0.95  &  Dendrogram purity=0.98 \\
       \begin{turn}{90} \quad \quad  Average-linkage \end{turn}&       
          \includegraphics[width=1.8in]{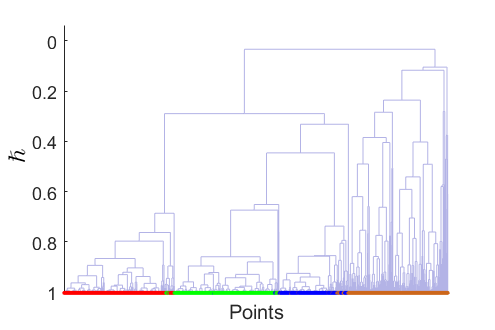}                           &
      \includegraphics[width=1.8in]{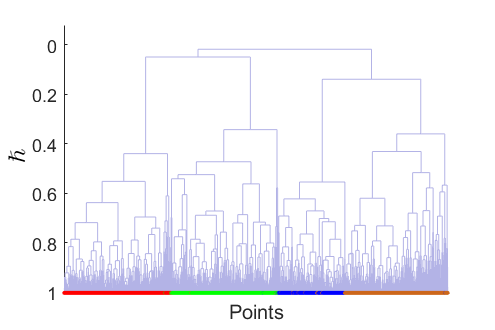}  \\    
            & No.  entanglements=103  &  No.  entanglements=100 \\
      & Avg.  entanglement level=1350 &Avg.  entanglement level=1445 \\
      \hdashline
      &Dendrogram purity=0.88   &  Dendrogram purity=0.97 \\
       \begin{turn}{90}\quad \quad  Complete-linkage  \end{turn}&
           \includegraphics[width=1.8in]{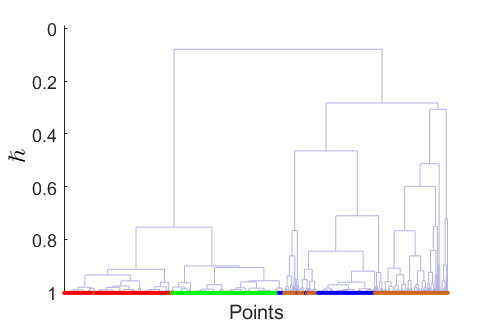}                               &
      \includegraphics[width=1.8in]{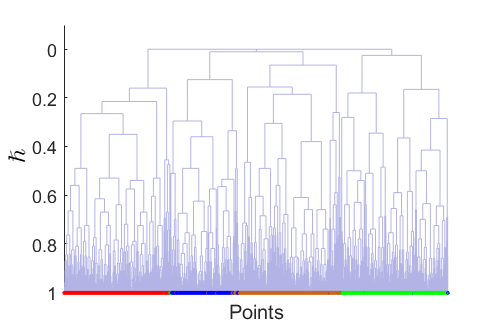}  \\   
                  & No.  entanglements=105   &  No.  entanglements=98 \\
            & Avg.  entanglement level=1394 &Avg.  entanglement level=1490 \\
            \hdashline
       &Dendrogram purity=0.91   &  Dendrogram purity=0.97 \\
       \begin{turn}{90}\qquad \qquad    Weighted \end{turn}&       
        \includegraphics[width=1.8in]{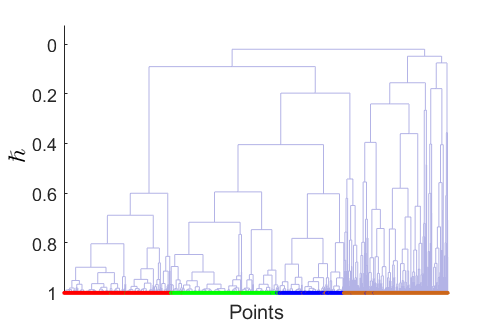}                  &
      \includegraphics[width=1.8in]{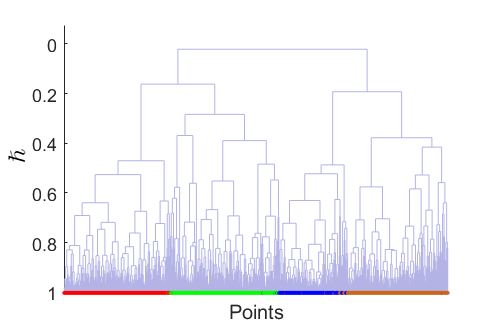} \\
                  & No.  entanglements=106  &  No.  entanglements=102 \\
            & Avg.  entanglement level=1360 &Avg.  entanglement level=1504 \\
            \\
      \hline
  \end{tabular}
  \caption{Comparison of dendrograms produced by T-AHC with four linkage functions using Gaussian Kernel and Isolation Kernel on the dataset shown in Figure \ref{fig:hard21}. The colours at the bottom row in each dendrogram correspond to the true cluster labels of all points shown in Figure \ref{fig:hard21}.}
  \label{fig:CompareDen}
\end{table}

\subsection{How data-dependent kernel helps}

Here we show that a data-dependent kernel called Isolation Kernel (IK) \citep{ting2018IsolationKernel,qin2019nearest}, which adapts its similarity measurement to the local density, significantly reduces the density bias posed by distance and data-independent kernel.

The unique characteristic of Isolation Kernel~\citep{ting2018IsolationKernel,qin2019nearest} is:
\textbf{two points in a sparse region are more similar than two points of equal inter-point distance in a dense region},
i.e,
$\forall x,y \in \mathcal{X}_{s}$, $\forall x^{\prime}, y^{\prime} \in \mathcal{X}_{d}$ if $\parallel x - y \parallel = \parallel x^{\prime} - y^{\prime}\parallel$ then
\begin{equation}
    K_{\psi}(x, y)>K_{\psi}\left(x^{\prime}, y^{\prime}\right)
\label{ike}
\end{equation}
where $\mathcal{X}_{s}$ and  $\mathcal{X}_{d}$ are two subsets of points in sparse and dense regions of $\mathbb{R}^d$, respectively; and $\parallel x - y \parallel$ is the distance between $x$ and $y$.\footnote{The proof of this characteristic is in the paper \cite{qin2019nearest}.}

Isolation Kernel deals more effectively with clusters with hugely different densities than data-independent kernels. This is because the isolation mechanism produces large partitions in a sparse region and small partitions in a dense region. Thus, the probability of two points from the dense cluster falling into the same isolating partition is higher than two points of equal inter-point distance from the sparse cluster. This gives rise to the data-dependent kernel characteristic mentioned above.

As a consequence of the above adaptation, the boundary points from a sparse cluster become less similar to the boundary points from a dense cluster. 
Thus, IK reduces the number of entanglements, and increases $s$ and $t$ in Equation \ref{eqn_entanglement} if an entanglement do occur in comparison with using a data-independent Kernel in T-AHC. 
%We demonstrate this effect of IK on both synthetic and real-world datasets in the following sections. 

Note that the condition in Equation \ref{eqn:condition} also applies when Isolation Kernel is used. However,  the influence of its data dependency is significant, as described below.

The isolation partitioning mechanism used to create IK is based on random samples from a given dataset. It yields two effects. First, the mechanism produces small partitions in dense clusters and large partitions in sparse clusters \citep{ting2018IsolationKernel,qin2019nearest}. The net effect is that every cluster has almost the same uniform distribution when transformed using MDS\footnote{Multidimensional scaling (MDS) is used for visualising the information contained in a similarity matrix \citep{borg2012applied}. It placed each data point in a $2$-dimensional space, while preserving as well as possible the pairwise similarities between points.} into a Euclidean space. This is shown in Figure \ref{fig:hard22}.
This effect was also observed in another data-dependent dissimilarity measure using a similar isolation mechanism, though it is not a kernel (see \cite{ting2019lowest} for details.)

Second, since data points in the valleys between clusters are less likely to include (than those within each cluster) in the sample used to create IK, each partition has a tiny or no chance to cover more than one cluster. As a result, the gaps/valleys between clusters become more pronounced in the transformed MDS space. Figure \ref{fig:hard22} shows a comparison between the MDS plots of GK and IK when they are used to transform the same dataset shown in Figure  \ref{fig:hard22}(a).

\begin{figure} [!htb]
     \centering
     \begin{subfigure}[b]{0.43\textwidth}
         \centering
         \includegraphics[width=\textwidth]{figure/Image/hard2.png}
           \caption{Dataset: 4 clusters of different densities}
         \label{fig:hard221}
     \end{subfigure}\\
     \begin{subfigure}[b]{0.43\textwidth}
         \centering
       \includegraphics[width=\textwidth]{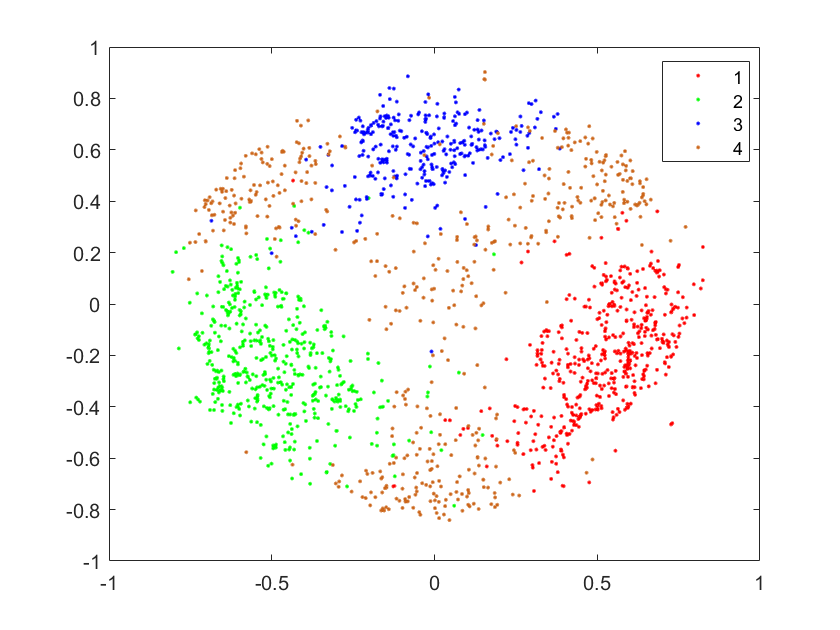}
         \caption{MDS plot using GK}
       \label{fig:hard222}
     \end{subfigure}
     \begin{subfigure}[b]{0.43\textwidth}
         \centering
       \includegraphics[width=\textwidth]{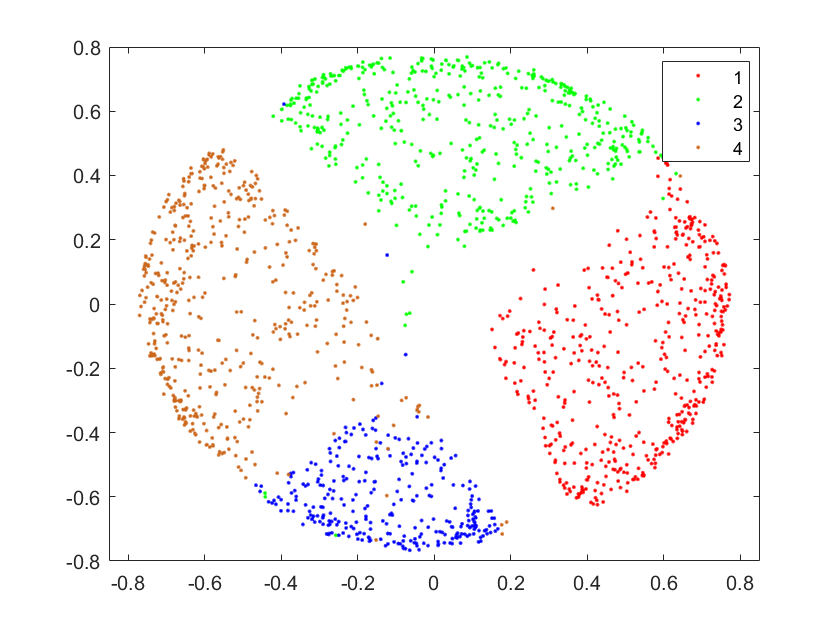}
         \caption{MDS plot using IK}
       \label{fig:hard223}
     \end{subfigure}
    \caption{MDS plot using Gaussian Kernel and Isolation Kernel on the dataset in (a).}
    \label{fig:hard22}
\end{figure}

These two net effects lead to the following observation:

\begin{observation}
T-AHC using a linkage function derived from Isolation Kernel has little or no bias towards points in dense regions, regardless of the relative density between clusters.
\end{observation}

In other words, the distribution of merges in the dendrogram produced by T-AHC becomes more uniform over all clusters, as a result of the first effect. Yet, the entanglements (i.e., merges between two different clusters) becomes less, as a consequence of the second effect.
This can be seen from all four dendrograms showed in Table \ref{fig:CompareDen} using four different linkage functions. The $\hslash$ values of every Isolation kernel-based linkage function are more uniformly distributed than those of the corresponding Gaussian kernel-based linkage function. Yet, the number of entanglements due to IK is less, and the average entanglement level is higher.

In summary, with Isolation Kernel: (a) an entanglement is less likely to occur since T-AHC always seeks to merge two subclusters which are most similar at each step; and (b) if an entanglement occurs, it will happen at higher values of $s$ and $t$ (in Equation \ref{eqn_entanglement}) than those due to Gaussian kernel. As a result,  the impact of incorrect membership assignment will be smaller than that when a data-independent kernel linkage function is used. 

This is verified using dendrogram purity in Equation \ref{eqn_dpurity} \citep{heller2005bayesian} which is an objective measure assessing the goodness of a dendrogram produced by an AHC algorithm. As shown in Table \ref{fig:CompareDen}, Isolation Kernel always has higher dendrogram purity than the Gaussian kernel in every linkage function. This result is also reflected in terms of the number of entanglements (the lower the better) and the average entanglement level (the higher the better), as stipulated in relation to Corollary \ref{eqn_entanglement}.

\subsection{Demonstration on an image segmentation task}

Here we use an image segmentation task to demonstrate the density bias of T-AHC on a real-world image, as shown in Figure \ref{image}. The scatter plot of this image in the LAB space \citep{szeliski2010computer} shows that there is a clear gap between a dense cluster (representing the sky object) which is in close proximity to a sparse cluster (representing the building object) in the LAB space, although the sparse cluster has some dense areas. The dendrogram produced by IK is much better than that by GK, judging from the dendrogram purity results. 

When using the single-linkage AHC with Gaussian Kernel, the dendrogram produces a result that the building object is partially merged to the sky object if the best cut is applied, as shown in the first row of Table \ref{seg2}. This is the effect of varied cluster densities in the LAB space.
%or (b) the building object (of the sparse cluster in the LAB space) will be separated into different segments if a low-value cut is used. However, 
In contrast, using the Isolation Kernel, the building and sky are well separated, as shown in the second row in Table \ref{seg2}.

\begin{figure}[!htb]
\centering
  \begin{subfigure}[b]{0.4\textwidth}
 \centering\captionsetup{width=0.8\linewidth}%
    \includegraphics[width=\textwidth]{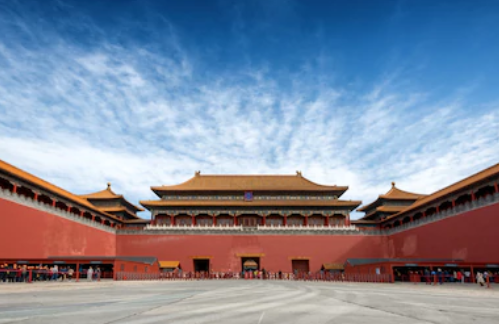}
    \caption{Image}
    \label{hard1:a}
  \end{subfigure}  %
  \begin{subfigure}[b]{0.4\textwidth}
 \centering\captionsetup{width=0.8\linewidth}%
    \includegraphics[width=\textwidth]{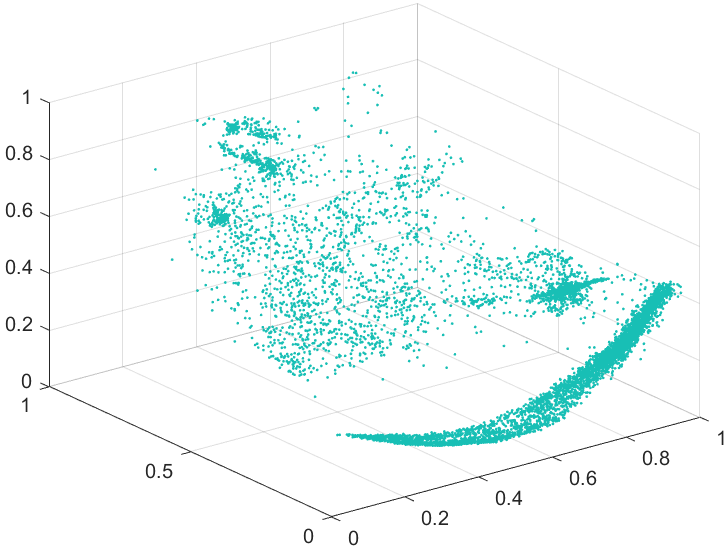}
    \caption{Pixels in LAB space}
        \label{hard1:d}
  \end{subfigure}
\caption{An example dataset of an image}
    \label{image} %% label for entire figure
\end{figure}

\begin{table}[!htb]
 \setlength{\tabcolsep}{1.8pt}
 \scriptsize
\centering
\caption{Image segmentation results on the image using the LAB space shown in Figure \ref{image}, produced from the AHC algorithm with either Gaussian Kernel or Isolation Kernel. Each scatter plot in the `LAB Space' column illustrates the top two clusters identified (indicated by different colours) where the building (mainly green and sparse) is separated from the sky (mainly blue and dense). Columns `Cluster 1' and `Cluster 2' show the segmentation results on the image. The colours at the bottom row in each dendrogram correspond to the true cluster labels of all points, i.e, blue for sky and green for building. 
}
  \begin{tabular}{ccccc}
    \hline
     &  LAB Space & Dendrogram  & Cluster 1   & Cluster 2  \\
      \hline
 %     \hdashline
       \begin{turn}{90} GK  single-linkage  \end{turn}&      \includegraphics[width=1.5in]{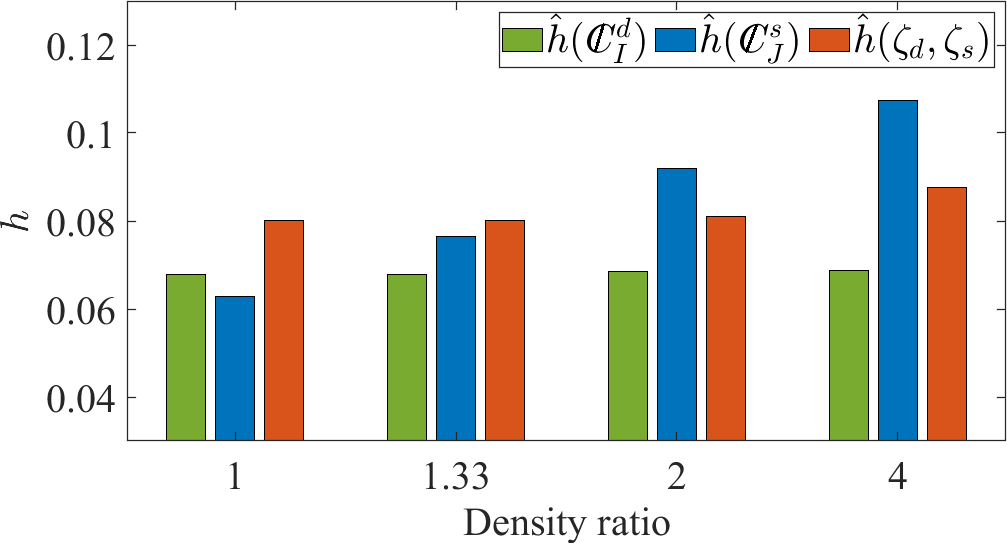} &
        \includegraphics[width=1.6in]{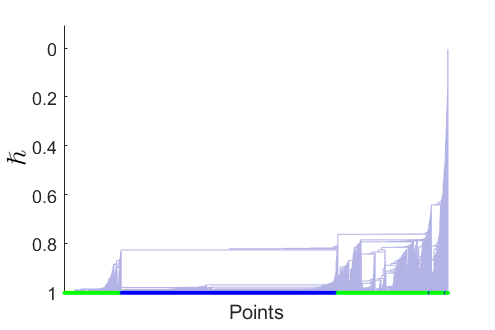} &
      \includegraphics[width=1.2in]{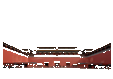} &
      \includegraphics[width=1.2in]{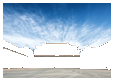}  \\ 
    &   Dendrogram purity=0.87 &  No. entanglements=210 & Avg. entanglement level=7640 \\
      \hdashline
       \begin{turn}{90} IK  single-linkage \end{turn}&      \includegraphics[width=1.5in]{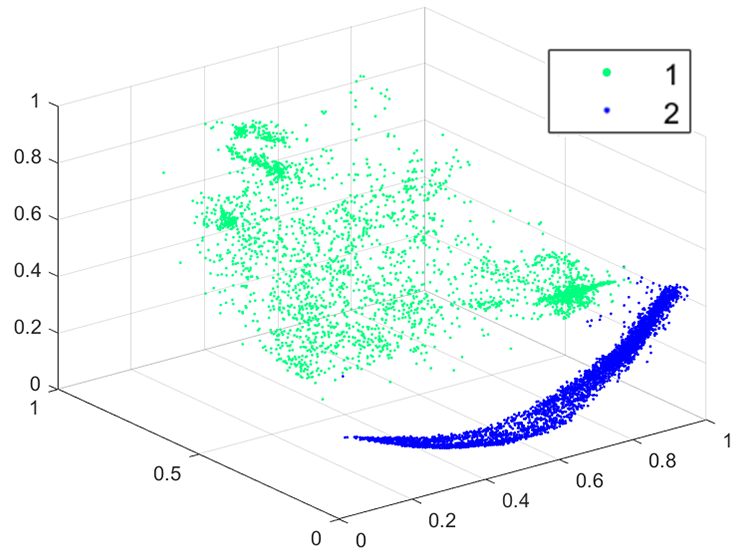} &
        \includegraphics[width=1.6in]{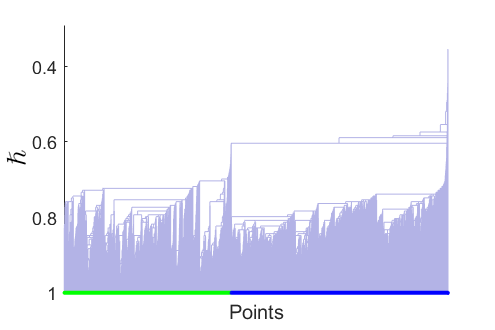} &
      \includegraphics[width=1.2in]{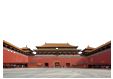} &
      \includegraphics[width=1.2in]{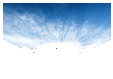}  \\
    &   Dendrogram purity=0.99 &  No. entanglements=16 & Avg. entanglement level=8091 \\ 
      \hline
  \end{tabular}
\label{seg2}
\end{table}

\subsection{Section summary}

We found that T-AHC using Gaussian Kernel has density bias, a known bias for T-AHC using distance. This bias heightens the severity of entanglements between clusters. As a result, T-AHC using Gaussian kernel will have difficulty separating the clusters successfully on a dataset with clusters of varied densities.
This phenomenon can be explained with the condition specified in Equation \ref{eqn:condition}.

We contend that the root cause of this bias is the use of a data-independent kernel. We show that using Isolation Kernel---because it adapts its similarity to local density---the resultant T-AHC has less density bias which reduces the severity and the number of entanglements. As a result, T-AHC using Isolation Kernel usually produces a better dendrogram than that using Gaussian Kernel.

We show in the next section that the same approach we suggested here, i.e., using Isolation Kernel instead of a data-dependent kernel/distance, can be applied to other hierarchical clustering algorithms, apart from T-AHC. We also show that not all data-dependent kernels are the same.

\section{Empirical evaluation}~\label{sec:experiments}

We provide the experimental settings and report the evaluation results in this section. In the experiment, we compared kernel-based hierarchical clustering with
traditional hierarchical clustering algorithms including
T-AHC with single-linkage and complete-linkage, the potential-based hierarchical agglomerative method PHA~\citep{lu2013pha}, the graph-based agglomerative method GDL~\citep{zhang2012graph} and the density-based method HDBSCAN~\citep{campello2013density}.
All algorithms used in our experiments are implemented in Matlab.

%where $P_i$ and $R_i$ are the \textit{precision} and the \textit{recall} for cluster $i$, respectively.

\subsection{Parameter settings}~\label{subsec:parameter}
Each parameter of an algorithm is searched within a certain range.
Table~\ref{para} shows the search ranges for all parameters; and we report the best performance on each dataset. The $k$ is the number of nearest neighbours for constructing the KNN graph and Adaptive Gaussian Kernel from Equation~\ref{Eqn_AG}. The $s$ in PHA is the scale factor for calculating the potential introduced \citep{lu2013pha}. $l$ and $c$ in HDBSCAN are minimum cluster sizes and minimum samples, respectively.\footnote{The guide  for parameter selection for HDBSCAN (so as its source code) is from https://hdbscan.readthedocs.io/ \citep{mcinnes2017accelerated}. All other codes used in empirical evaluations are published by the original authors. It is worth noting that not all parameters are valid for GDL in all datasets because the based $k$-nearest neighbours graph methods are sensitive for the parameter and similarity measure. Here we only record the available results. In addition, GDL source code does not output the dendrogram, thus we have to omit it in the dendrogram purity evaluation.}

\begin{table}[!htb]
  \centering
    \footnotesize
  \caption{Parameters and their search ranges for each algorithm. For AGK and IK, we searched all integer values within $[2,\lceil n/2 \rceil]$. T-AHC is an AHC using one of the four linkage functions in Table \ref{Kernel}.}
    \begin{tabular}{|c|c|}
    \hline
    Algorithm/Kernel & Parameter search range \\
    \hline
    T-AHC & $\kappa \in \{2,\dots,30\}$ \\
    HDBSCAN & $l \in [2,100]$, $c \in [2,100] $ \\
%    CLR & $k \in \{5,10,15,20,25,30,70,100\}$\\
    PHA & $s \in \{5,10,15,20,25,30\}$ \\
    GDL & $k \in \{5,10,15,20,25,30,70,100\}$ \\
\hdashline
Gaussian Kernel &  $\sigma = 2^m$, \newline{} $m \in [-5,5]$  \\
Adaptive Gaussian Kernel &  $k  \in [2, \lceil n/2 \rceil]$ \\
Isolation Kernel &  $\psi  \in [2, \lceil n/2 \rceil]$, $t=200$ \\
    \hline
    \end{tabular}%
  \label{para}%
\end{table}%

The experiments use 19 real-world datasets with different data sizes and dimensions from UCI Machine Learning Repository \citep{Dua:2019}. The data properties are shown in the first four columns in Table~\ref{tab:dataset}.
All datasets were normalised using the Min-Max normalisation to yield each attribute to be in [0,1] before the experiments began. Some of these datasets have been shown to have clusters with varied densities, e.g., thyroid, seeds, wine, WDBC and Segment~\citep{zhu2016density}.

% Table generated by Excel2LaTeX from sheet 'purity'
\begin{table}[!htb]
\renewcommand{\arraystretch}{0.8}
  \setlength{\tabcolsep}{2pt}
  \centering
  \caption{Properties of the datasets used in the experiments}
    \begin{tabular}{|c|ccc|}
    \hline
    Name & \#instance & \#Dim. & \#Clusters \\
    \hline
    banknote & 1372 & 4 & 2 \\
    thyroid & 215 & 5 & 3 \\
    seeds & 210 & 7 & 3 \\
    diabetes & 768 & 8 & 2 \\
    vowel & 990 & 10 & 11 \\
    wine & 178 & 13 & 3 \\
    shape & 160 & 17 & 9 \\
    Segment & 2310 & 19 & 7 \\
    WDBC & 569 & 30 & 2 \\
    spam & 4601 & 57 & 2 \\
    control & 600 & 60 & 6 \\
    hill & 1212 & 100 & 2 \\
    LandCover & 675 & 147 & 9 \\
    musk & 476 & 166 & 2 \\
    LSVT & 126 & 310 & 2 \\
    Isolet & 1560 & 617 & 26 \\
    COIL20 & 1440 & 1024 & 20 \\
    lung & 203 & 3312 & 5 \\
    ALLAML & 72 & 7129 & 2 \\
    \hline
    \end{tabular}%
  \label{tab:dataset}%
\end{table}%

\subsection{Hierarchical clustering evaluation}

We evaluate cluster trees using \textit{Dendrogram Purity} shown in Equation~\ref{eqn_dpurity}.
In words,
to compute \textit{Dendrogram Purity} of a dendrogram $\mathcal{T}$ with ground-truth clusters $C^{\star}$,
for all pairs of points $\left(x_i, x_j \right)$ which belong to the same ground-truth cluster,
find the smallest subdendrogram containing $x_i$ and $x_j$,
and compute the fraction of leaves in that subdendrogram which are in the same ground-truth cluster as $x_i$ and $x_j$.
For large-scale dataset,
we use Monte Carlo to approximated the \textit{Dendrogram Purity}.

In Table~\ref{tab:purity},
we report the dendrogram purity for the two T-AHC algorithms (single-linkage,
average-linkage),
PHA,
HDBSCAN
and their kernerlise with three kernel functions (G,AG and IK).
The best result on each dataset is
boldfaced.
We observe that kernel method consistently produces dendrogram with the highest \textit{dendrogram purity} amongst all algorithms.

\begin{landscape}
\begin{table}[!htb]
\small
 \renewcommand{\arraystretch}{1}
  \setlength{\tabcolsep}{1pt}
  \centering
  \caption{Clustering results in \textit{Dendrogram Purity}. The best result  on each dataset is boldfaced. $n$, $d$, $\kappa$ are \#Point, \#Dimension and \#Clusters, respectively.} %\textcolor{red}{Move or get rid of Complete- \& Weighted-linkage since they are weaker. Either mention in the main text or move their results to the appendix.}}
    \begin{tabular}{|c|cccc|cccc|cccc|cccc|cccc|cccc|}
    \hline
    \multicolumn{1}{|c|}{\multirow{2}[4]{*}{Dataset}} & \multicolumn{4}{c|}{Single-linkage AHC} & \multicolumn{4}{c|}{Complete-linkage AHC} & \multicolumn{4}{c|}{Average-linkage AHC} & \multicolumn{4}{c|}{Weight-linkage AHC} & \multicolumn{4}{c|}{PHA} & \multicolumn{4}{c|}{HDBSCAN} \\
\cline{2-25}      & Dis & G & AG & IK & Dis & G & AG & IK & Dis & G & AG & IK & Dis & G & AG & IK & Dis & G & AG & IK & Dis & G & AG & IK \\
    \hline
    \multicolumn{1}{|c|}{banknote} & 0.92 & 0.92 & \textbf{0.99} & \textbf{0.99} & 0.63 & 0.63 & 0.80 & \textbf{0.82} & 0.68 & 0.96 & 0.78 & \textbf{0.98} & 0.64 & 0.78 & 0.73 & \textbf{0.94} & 0.62 & 0.71 & 0.68 & \textbf{0.76} & 0.92 & 0.92 & 0.96 & \textbf{0.99} \\
    \multicolumn{1}{|c|}{thyroid} & 0.92 & 0.92 & 0.93 & \textbf{0.93} & 0.89 & 0.89 & 0.95 & \textbf{0.97} & 0.93 & 0.92 & 0.94 & \textbf{0.96} & 0.86 & 0.92 & 0.95 & \textbf{0.97} & 0.91 & 0.92 & 0.92 & \textbf{0.93} & 0.92 & 0.93 & 0.92 & \textbf{0.96} \\
    \multicolumn{1}{|c|}{seeds} & 0.69 & 0.69 & 0.81 & \textbf{0.85} & 0.75 & 0.75 & 0.84 & \textbf{0.86} & 0.85 & 0.85 & 0.88 & \textbf{0.89} & 0.76 & 0.74 & 0.85 & \textbf{0.88} & 0.84 & 0.84 & \textbf{0.89} & 0.88 & 0.65 & 0.73 & 0.79 & \textbf{0.84} \\
    \multicolumn{1}{|c|}{diabetes} & 0.67 & 0.67 & 0.67 & \textbf{0.70} & 0.66 & 0.66 & 0.67 & \textbf{0.68} & 0.67 & 0.67 & 0.67 & \textbf{0.70} & 0.63 & 0.64 & 0.66 & \textbf{0.67} & 0.67 & 0.68 & 0.68 & \textbf{0.69} & 0.68 & 0.68 & 0.67 & \textbf{0.69} \\
    \multicolumn{1}{|c|}{vowel} & 0.20 & 0.20 & 0.24 & \textbf{0.26} & 0.22 & 0.22 & 0.25 & \textbf{0.27} & 0.22 & 0.22 & 0.24 & \textbf{0.28} & 0.23 & 0.25 & 0.27 & \textbf{0.30} & 0.20 & 0.22 & 0.23 & \textbf{0.24} & 0.19 & 0.19 & 0.23 & \textbf{0.24} \\
    \multicolumn{1}{|c|}{wine} & 0.68 & 0.68 & 0.84 & \textbf{0.90} & 0.92 & 0.92 & 0.96 & \textbf{0.98} & 0.89 & 0.92 & 0.94 & \textbf{0.96} & 0.81 & 0.82 & 0.93 & \textbf{0.94} & 0.73 & 0.74 & 0.91 & \textbf{0.94} & 0.63 & 0.79 & 0.75 & \textbf{0.89} \\
    \multicolumn{1}{|c|}{shape} & 0.69 & 0.69 & \textbf{0.70} & \textbf{0.70} & 0.65 & 0.65 & 0.68 & \textbf{0.69} & 0.65 & 0.65 & \textbf{0.72} & \textbf{0.72} & 0.68 & 0.68 & 0.73 & \textbf{0.74} & 0.67 & 0.67 & 0.70 & \textbf{0.74} & 0.68 & 0.70 & 0.70 & \textbf{0.75} \\
    \multicolumn{1}{|c|}{Segment} & 0.60 & 0.60 & 0.64 & \textbf{0.67} & 0.62 & 0.62 & 0.65 & \textbf{0.67} & 0.65 & 0.67 & 0.69 & \textbf{0.72} & 0.59 & 0.61 & 0.66 & \textbf{0.71} & 0.65 & 0.70 & 0.72 & \textbf{0.74} & 0.59 & 0.59 & 0.64 & \textbf{0.67} \\
    \multicolumn{1}{|c|}{WDBC} & 0.71 & 0.72 & 0.74 & \textbf{0.90} & 0.79 & 0.79 & 0.89 & \textbf{0.91} & 0.86 & 0.89 & 0.89 & \textbf{0.93} & 0.83 & 0.86 & 0.90 & \textbf{0.94} & 0.73 & 0.79 & 0.87 & \textbf{0.95} & 0.73 & 0.73 & 0.71 & \textbf{0.90} \\
    \multicolumn{1}{|c|}{spam} & 0.59 & 0.56 & 0.57 & \textbf{0.64} & 0.57 & 0.56 & 0.60 & \textbf{0.69} & 0.58 & 0.56 & 0.64 & \textbf{0.69} & 0.58 & 0.59 & 0.66 & \textbf{0.69} & 0.55 & 0.57 & 0.65 & \textbf{0.68} & 0.55 & 0.56 & 0.57 & \textbf{0.94} \\
    \multicolumn{1}{|c|}{control} & 0.73 & 0.73 & 0.80 & \textbf{0.87} & 0.81 & 0.81 & 0.82 & \textbf{0.85} & 0.81 & 0.81 & 0.91 & \textbf{0.94} & 0.77 & 0.75 & 0.89 & \textbf{0.90} & 0.66 & 0.67 & 0.79 & \textbf{0.82} & 0.71 & 0.71 & 0.75 & \textbf{0.83} \\
    \multicolumn{1}{|c|}{hill} & 0.50 & 0.50 & \textbf{0.51} & \textbf{0.51} & 0.50 & 0.50 & \textbf{0.51} & \textbf{0.51} & 0.50 & 0.50 & \textbf{0.51} & \textbf{0.51} & 0.50 & 0.50 & \textbf{0.51} & \textbf{0.51} & 0.50 & 0.50 & \textbf{0.51} & \textbf{0.51} & 0.50 & 0.50 & \textbf{0.51} & \textbf{0.51} \\
    \multicolumn{1}{|c|}{LandCover} & 0.30 & 0.30 & 0.39 & \textbf{0.55} & 0.52 & 0.52 & 0.58 & \textbf{0.60} & 0.56 & 0.59 & 0.63 & \textbf{0.64} & 0.48 & 0.56 & 0.59 & \textbf{0.60} & 0.44 & 0.48 & 0.53 & \textbf{0.61} & 0.27 & 0.27 & 0.35 & \textbf{0.48} \\
    \multicolumn{1}{|c|}{musk} & 0.54 & \textbf{0.64} & 0.54 & 0.55 & 0.56 & \textbf{0.65} & 0.56 & 0.57 & 0.55 & \textbf{0.65} & 0.56 & 0.57 & 0.55 & \textbf{0.65} & 0.56 & 0.57 & 0.54 & \textbf{0.65} & 0.56 & 0.56 & 0.54 & \textbf{0.94} & 0.56 & 0.93 \\
    \multicolumn{1}{|c|}{LSVT} & 0.58 & 0.60 & 0.63 & \textbf{0.67} & 0.62 & 0.62 & 0.65 & \textbf{0.71} & 0.63 & 0.63 & 0.65 & \textbf{0.67} & 0.61 & 0.64 & 0.66 & \textbf{0.69} & 0.59 & 0.60 & 0.67 & \textbf{0.71} & 0.58 & 0.61 & 0.62 & \textbf{0.64} \\
    \multicolumn{1}{|c|}{Isolet} & 0.27 & 0.27 & 0.36 & \textbf{0.58} & 0.48 & 0.48 & 0.56 & \textbf{0.57} & 0.57 & 0.60 & 0.62 & \textbf{0.68} & 0.54 & 0.59 & 0.59 & \textbf{0.66} & 0.39 & 0.48 & 0.50 & \textbf{0.64} & 0.21 & 0.57 & 0.29 & \textbf{0.61} \\
    \multicolumn{1}{|c|}{COIL20} & 0.91 & 0.91 & 0.99 & \textbf{1.00} & 0.66 & 0.67 & 0.66 & \textbf{0.70} & 0.72 & \textbf{0.93} & 0.79 & \textbf{0.93} & 0.72 & 0.97 & 0.82 & \textbf{0.98} & 0.70 & 0.76 & 0.76 & \textbf{0.79} & 0.89 & 0.90 & 0.99 & \textbf{1.00} \\
    \multicolumn{1}{|c|}{lung} & 0.76 & 0.91 & 0.82 & \textbf{0.95} & 0.89 & 0.92 & 0.93 & \textbf{0.96} & 0.94 & 0.94 & \textbf{0.96} & \textbf{0.96} & 0.94 & 0.95 & 0.95 & \textbf{0.96} & 0.77 & 0.95 & 0.86 & \textbf{0.98} & 0.87 & 0.87 & 0.90 & \textbf{0.97} \\
    \multicolumn{1}{|c|}{ALLAML} & 0.68 & 0.68 & 0.68 & \textbf{0.72} & 0.67 & 0.68 & 0.73 & \textbf{0.74} & 0.68 & 0.69 & 0.72 & \textbf{0.74} & 0.67 & 0.70 & 0.72 & \textbf{0.79} & 0.68 & 0.72 & 0.72 & \textbf{0.77} & 0.66 & \textbf{0.74} & 0.70 & \textbf{0.74} \\
    \hline
    Average & 0.63 & 0.64 & 0.68 & 0.73 & 0.65 & 0.66 & 0.70 & 0.72 & 0.68 & 0.72 & 0.72 & 0.76 & 0.65 & 0.70 & 0.72 & 0.76 & 0.62 & 0.67 & 0.69 & 0.73 & 0.62 & 0.68 & 0.66 & 0.77 \\
    \hline
    \end{tabular}%
   \label{tab:purity}%
 \end{table}%
\end{landscape}

We conduct a Friedman test with the post-hoc Nemenyi test~\citep{demvsar2006statistical} to examine whether the difference in \textit{Dendrogram Purity} of any two measures is significant.
Those four measures (used in an algorithm) are ranked based on their \textit{Dendrogram Purity} on each dataset,
where the best one is rank $1$ and so on.
Then the critical difference (CD) is computed using the post-hoc Nemenyi test. Two measures are significantly different if the difference in their average ranks is larger than CD.
Figure \ref{fig:NemenyitestP} shows that IK is significantly better than all other measures for every algorithm.

\begin{figure}[t]
     \centering
     \begin{subfigure}[b]{0.3\textwidth}
         \centering
         \includegraphics[width=\textwidth]{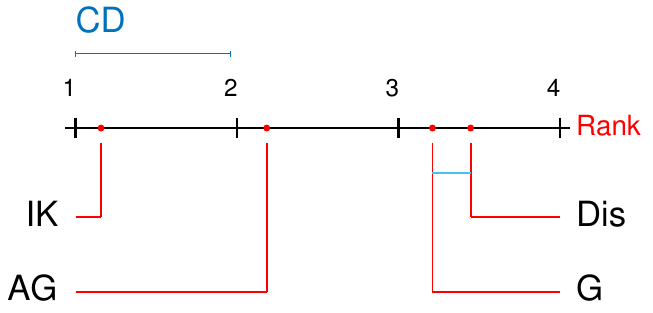}
         \caption{Single-linkage AHC}
         \label{NemenyitestP:single}
     \end{subfigure}
    \begin{subfigure}[b]{0.3\textwidth}
       \centering
       \includegraphics[width=\textwidth]{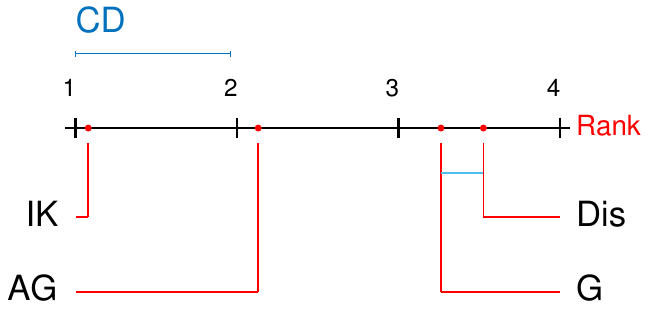}
      \caption{Complete-linkage AHC}
       \label{NemenyitestP:complete}
    \end{subfigure}
     \begin{subfigure}[b]{0.3\textwidth}
         \centering
         \includegraphics[width=\textwidth]{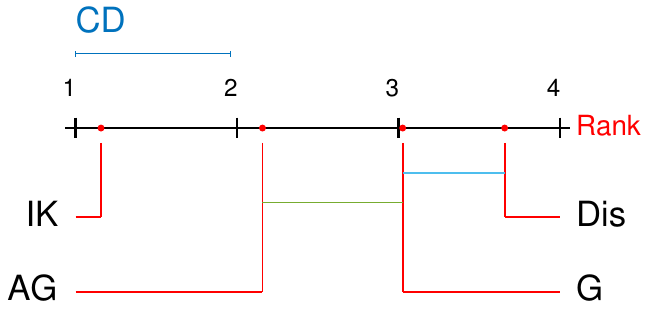}
         \caption{Average-linkage AHC}
         \label{NemenyitestP:average}
     \end{subfigure}
     \begin{subfigure}[b]{0.3\textwidth}
         \centering
         \includegraphics[width=\textwidth]{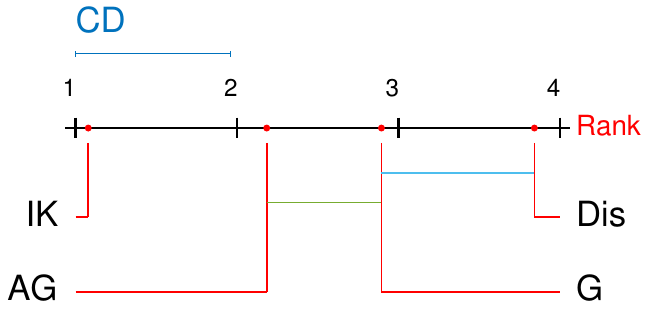}
         \caption{Weighted-linkage AHC}
         \label{NemenyitestP:weighted}
     \end{subfigure}
     \begin{subfigure}[b]{0.3\textwidth}
         \centering
         \includegraphics[width=\textwidth]{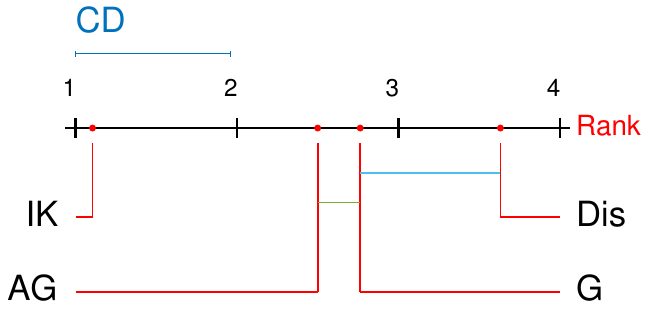}
         \caption{HDBSCAN}
         \label{NemenyitestP:hdbscan}
     \end{subfigure}
     \begin{subfigure}[b]{0.3\textwidth}
         \centering
         \includegraphics[width=\textwidth]{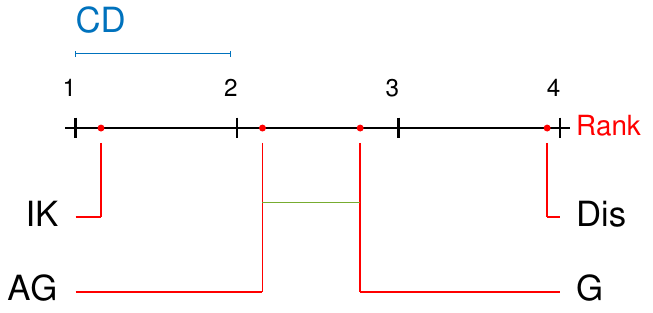}
         \caption{PHA}
         \label{NemenyitestP:pha}
     \end{subfigure}
        \caption{Critical difference (CD) diagram of the post-hoc Nemenyi test ($\alpha=0.10$) for dendrogram purity. Two measures are not significantly different if there is a line linking them.}
        \label{fig:NemenyitestP}
\end{figure}

\subsection{Flat Clustering Evaluation} 

To evaluate the flat clustering results, we use the original cluster extraction method in each algorithm, and compare the extracted clusters with ground truth cluster using $F1$ score which is a trade-off between the
\textit{Precision}
and
\textit{Recall}~\citep{Rijsbergen1979,larsen1999fast,aliguliyev2009performance}. Note that we use a global cut for T-AHC to extract the $k$ subclusters on the dendrogram.

Given a clustering result,
the \textit{precision} score $P_{i}$ and the \textit{recall} score $R_{i}$ for
each cluster $C_{i}$ are calculated based on the confusion matrix,
and the $F1$ score of $C_{i}$ is the harmonic mean of $P_{i}$ and $R_{i}$.
The Hungarian algorithm~\citep{kuhn1955hungarian} is used to search the
optimal match be between the clustering results and
ground-truth clusters.
The overall $F1$ score is the unweighted average
overall matched clusters as:
\begin{equation}\label{eq:f1}
  F1 = \frac{1}{k}\sum_{i=1}^{k}\frac{2 \times P_i \times R_i}{P_i+R_i}
\end{equation}

Note that other evaluation measures such as Purity \citep{manning2010introduction} and Adjusted Mutual Information \citep{vinh2010information} do not take into account noise points identified by a clustering algorithm.
They are not suitable for HDBSCAN in the evaluation section because these scores can provide a misleadingly good clustering result when HDBSCAN assigns many points to noise.

Table~\ref{tab:perf} reports the experimental results of traditional
AHC and other three algorithms.
The key observations are:
\begin{itemize}
    \item \textbf{Kernel improves the flat clustering performance of distance measure}.
        Every clustering algorithm which employs a kernel has better or equivalent clustering performance than that using distance in terms of average $F1$ score shown in the last row in Table~\ref{tab:perf}.
    The only exception is GDL, where Gaussian Kernel is marginally worse than distance. Even in GDL, IK and AGK are always better than distance (except on thyroid and lung only).
    \item \textbf{Among three kernels, Isolation Kernel (IK) produces the best $F1$ Score}. This occurs on all datasets, and on every algorithm. The only exception is with complete-linkage AHC on musk where IK is better than AGK but worse than GK).
    IK has the best $F1$ on 17-18 out of the 19 datasets for every clustering algorithm. The best contender is AGK which has the best $F1$ on 1-3 datasets on four algorithms, but it produced no best $F1$ on complete-linkage AHC.% \textcolor{red}{Please check the numbers are correct.}
    \item \textbf{IK produced a huge performance gap in $F1$ on some datasets}, even compared with other kernels. For example, (i) tyroid, WDBC, LSVT, Isolet and lung using single-linkage AHC; (ii) thyroid, spam, control, musk and ALLAML using HDBSCAN; (iii) LandCover and Isolet using PHA.
\end{itemize}

It is interesting to note that Isolation Kernel achieves the
largest performance improvement over
distance on almost all datasets for all algorithms.
In addition,
using IK in Complete-linkage AHC, PHA and GDL allow
all these algorithms to produce similar average $F1$ score.
In contrast,
using AGK only allows Complete-linkage and
GDL to produce similar average $F1$ score;
and using GK in these two algorithms makes
little difference or worse average $F1$ score
than those using distance.

\begin{landscape}
\begin{table}[!htb]
\small
 \renewcommand{\arraystretch}{1}
  \setlength{\tabcolsep}{1.2pt}
  \centering
  \caption{Clustering results in $F1$ score. A larger $F1$ score indicates a better clustering result. The best result  on each dataset is boldfaced.}  
    \begin{tabular}{|c|cccc|cccc|cccc|cccc|cccc|cccc|cccc|}
    \hline
    \multirow{2}[0]{*}{Dataset}  & \multicolumn{4}{c|}{Single-linkage AHC} & \multicolumn{4}{c|}{Complete-linkage AHC} & \multicolumn{4}{c|}{Average-l AHC} & \multicolumn{4}{c|}{Weighted-l AHC} & \multicolumn{4}{c|}{PHA} & \multicolumn{4}{c|}{HDBSCAN} & \multicolumn{4}{c|}{GDL} \\
    \cline{2-29}
      & Dis & G & AG & IK & Dis & G & AG & IK & Dis & G & AG & IK & Dis & G & AG & IK & Dis & G & AG & IK & Dis & G & AG & IK & Dis & G & AG & IK \\  \hline
    \multicolumn{1}{|c|}{banknote} & 0.95 & 0.95 & \textbf{0.99} & \textbf{0.99} & 0.60 & 0.60 & 0.79 & \textbf{0.81} & 0.61 & 0.97 & 0.77 & \textbf{0.99} & 0.62 & 0.86 & 0.73 & \textbf{0.94} & 0.68 & 0.74 & 0.78 & \textbf{0.96} & 0.69 & 0.69 & 0.94 & \textbf{0.95} & 0.98 & 0.91 & \textbf{0.99} & \textbf{0.99} \\
    \multicolumn{1}{|c|}{thyroid} & 0.58 & 0.58 & 0.76 & \textbf{0.91} & 0.80 & 0.80 & 0.95 & \textbf{0.96} & 0.73 & 0.73 & \textbf{0.95} & 0.91 & 0.75 & 0.75 & 0.93 & \textbf{0.95} & 0.55 & 0.59 & 0.85 & \textbf{0.86} & 0.57 & 0.59 & 0.58 & \textbf{0.78} & \textbf{0.91} & 0.73 & 0.71 & 0.84 \\
    \multicolumn{1}{|c|}{seeds} & 0.54 & 0.54 & 0.90 & \textbf{0.91} & 0.85 & 0.85 & 0.89 & \textbf{0.94} & 0.89 & 0.89 & 0.93 & \textbf{0.94} & 0.81 & 0.85 & 0.91 & \textbf{0.92} & 0.90 & 0.90 & 0.92 & \textbf{0.93} & 0.61 & 0.58 & 0.82 & \textbf{0.83} & 0.88 & 0.88 & 0.92 & \textbf{0.93} \\
    \multicolumn{1}{|c|}{diabetes} & 0.40 & 0.40 & 0.44 & \textbf{0.50} & 0.53 & 0.53 & 0.68 & \textbf{0.70} & 0.61 & 0.61 & 0.64 & \textbf{0.65} & 0.49 & 0.59 & 0.69 & \textbf{0.65} & 0.44 & 0.59 & 0.58 & \textbf{0.63} & 0.42 & 0.47 & 0.43 & \textbf{0.52} & 0.53 & 0.52 & 0.59 & \textbf{0.63} \\
    \multicolumn{1}{|c|}{vowel} & 0.08 & 0.12 & 0.27 & \textbf{0.31} & 0.29 & 0.29 & 0.33 & \textbf{0.34} & 0.28 & 0.31 & 0.33 & \textbf{0.37} & 0.27 & 0.32 & 0.35 & \textbf{0.37} & 0.12 & 0.25 & 0.33 & \textbf{0.35} & 0.24 & 0.27 & 0.32 & \textbf{0.33} & 0.31 & 0.30 & 0.31 & \textbf{0.36} \\
    \multicolumn{1}{|c|}{wine} & 0.56 & 0.57 & \textbf{0.92} & \textbf{0.92} & 0.94 & 0.94 & 0.95 & \textbf{0.98} & 0.93 & 0.96 & 0.96 & \textbf{0.97} & 0.87 & 0.92 & 0.96 & \textbf{0.98} & 0.58 & 0.59 & 0.93 & \textbf{0.95} & 0.51 & 0.55 & 0.82 & \textbf{0.91} & 0.92 & 0.92 & \textbf{0.97} & \textbf{0.97} \\
    \multicolumn{1}{|c|}{shape} & 0.44 & 0.52 & 0.75 & \textbf{0.77} & 0.64 & 0.66 & 0.74 & \textbf{0.78} & 0.63 & 0.68 & \textbf{0.76} & \textbf{0.76} & 0.67 & 0.68 & 0.74 & \textbf{0.77} & 0.64 & 0.65 & \textbf{0.76} & \textbf{0.76} & 0.68 & 0.68 & \textbf{0.71} & \textbf{0.71} & 0.70 & 0.66 & 0.71 & \textbf{0.74} \\
    \multicolumn{1}{|c|}{Segment} & 0.34 & 0.36 & 0.55 & \textbf{0.66} & 0.65 & 0.65 & 0.73 & \textbf{0.75} & 0.53 & 0.58 & 0.73 & \textbf{0.79} & 0.55 & 0.60 & 0.72 & \textbf{0.76} & 0.39 & 0.71 & 0.73 & \textbf{0.75} & 0.61 & 0.60 & \textbf{0.69} & 0.63 & 0.42 & 0.45 & 0.78 & \textbf{0.82} \\
    \multicolumn{1}{|c|}{WDBC} & 0.40 & 0.67 & 0.40 & \textbf{0.93} & 0.77 & 0.77 & 0.92 & \textbf{0.95} & 0.87 & 0.91 & 0.91 & \textbf{0.95} & 0.79 & 0.88 & 0.92 & \textbf{0.96} & 0.40 & 0.41 & 0.88 & \textbf{0.94} & 0.41 & 0.54 & 0.71 & \textbf{0.91} & 0.94 & 0.94 & 0.95 & \textbf{0.96} \\
    \multicolumn{1}{|c|}{spam} & 0.38 & 0.38 & 0.38 & \textbf{0.42} & 0.42 & 0.51 & 0.64 & \textbf{0.73} & 0.38 & 0.38 & 0.66 & \textbf{0.77} & 0.39 & 0.52 & 0.74 & \textbf{0.80} & 0.38 & 0.38 & 0.70 & \textbf{0.71} & 0.28 & 0.35 & 0.48 & \textbf{0.87} & 0.38 & 0.38 & 0.69 & \textbf{0.74} \\
    \multicolumn{1}{|c|}{control} & 0.23 & 0.23 & 0.60 & \textbf{0.79} & 0.75 & 0.75 & 0.84 & \textbf{0.86} & 0.71 & 0.73 & 0.86 & \textbf{0.87} & 0.72 & 0.74 & 0.82 & \textbf{0.84} & 0.46 & 0.63 & 0.73 & \textbf{0.81} & 0.32 & 0.57 & 0.58 & \textbf{0.76} & 0.76 & 0.76 & \textbf{0.95} & \textbf{0.95} \\
    \multicolumn{1}{|c|}{hill} & 0.34 & 0.37 & \textbf{0.51} & 0.46 & 0.37 & 0.40 & 0.50 & \textbf{0.51} & 0.37 & 0.41 & \textbf{0.51} & \textbf{0.51} & 0.39 & 0.40 & 0.51 & \textbf{0.52} & 0.37 & 0.39 & 0.54 & \textbf{0.56} & 0.33 & 0.36 & 0.47 & \textbf{0.48} & 0.46 & 0.35 & 0.50 & \textbf{0.62} \\
    \multicolumn{1}{|c|}{LandCover} & 0.16 & 0.08 & 0.18 & \textbf{0.36} & 0.59 & 0.60 & 0.64 & \textbf{0.65} & 0.36 & 0.55 & 0.72 & \textbf{0.74} & 0.58 & 0.67 & 0.67 & \textbf{0.72} & 0.08 & 0.41 & 0.49 & \textbf{0.70} & 0.18 & 0.28 & 0.39 & \textbf{0.55} & 0.62 & 0.61 & 0.74 & \textbf{0.75} \\
    \multicolumn{1}{|c|}{musk} & 0.36 & \textbf{0.64} & 0.50 & 0.53 & \textbf{0.58} & \textbf{0.58} & 0.55 & 0.57 & 0.36 & 0.53 & 0.53 & \textbf{0.56} & 0.48 & \textbf{0.54} & \textbf{0.54} & \textbf{0.54} & 0.51 & 0.51 & 0.55 & \textbf{0.57} & 0.50 & 0.52 & 0.50 & \textbf{0.75} & 0.48 & 0.48 & 0.50 & \textbf{0.56} \\
    \multicolumn{1}{|c|}{LSVT} & 0.40 & 0.49 & 0.44 & \textbf{0.67} & 0.40 & 0.55 & 0.58 & \textbf{0.65} & 0.40 & 0.60 & 0.58 & \textbf{0.62} & 0.40 & 0.61 & 0.60 & \textbf{0.62} & 0.40 & 0.42 & 0.67 & \textbf{0.68} & 0.15 & 0.42 & 0.42 & \textbf{0.59} & 0.51 & 0.51 & 0.60 & \textbf{0.63} \\
    \multicolumn{1}{|c|}{Isolet} & 0.03 & 0.03 & 0.07 & \textbf{0.28} & 0.33 & 0.39 & 0.59 & \textbf{0.68} & 0.12 & 0.30 & 0.60 & \textbf{0.68} & 0.24 & 0.40 & 0.59 & \textbf{0.66} & 0.03 & 0.24 & 0.31 & \textbf{0.66} & 0.09 & 0.35 & 0.38 & \textbf{0.55} & 0.61 & 0.61 & 0.67 & \textbf{0.73} \\
    \multicolumn{1}{|c|}{COIL20} & 0.28 & 0.33 & 0.96 & \textbf{0.97} & 0.39 & 0.46 & 0.71 & \textbf{0.73} & 0.24 & 0.56 & 0.77 & \textbf{0.90} & 0.27 & 0.71 & 0.80 & \textbf{0.92} & 0.44 & 0.49 & 0.73 & \textbf{0.76} & 0.84 & 0.84 & \textbf{0.95} & \textbf{0.95} & 0.86 & 0.86 & \textbf{0.87} & \textbf{0.87} \\
    \multicolumn{1}{|c|}{lung} & 0.43 & 0.43 & 0.52 & \textbf{0.88} & 0.81 & 0.81 & 0.85 & \textbf{0.86} & 0.87 & 0.87 & 0.87 & \textbf{0.90} & 0.87 & 0.91 & 0.88 & \textbf{0.94} & 0.42 & 0.61 & 0.78 & \textbf{0.96} & 0.23 & 0.58 & 0.68 & \textbf{0.76} & 0.91 & 0.91 & 0.90 & \textbf{0.94} \\
    \multicolumn{1}{|c|}{ALLAML} & 0.46 & 0.47 & 0.61 & \textbf{0.75} & 0.62 & 0.62 & 0.74 & \textbf{0.75} & 0.57 & 0.61 & 0.74 & \textbf{0.75} & 0.53 & 0.63 & 0.72 & \textbf{0.82} & 0.46 & 0.49 & 0.72 & \textbf{0.75} & 0.36 & 0.49 & 0.53 & \textbf{0.73} & 0.60 & 0.60 & 0.72 & \textbf{0.78} \\
    \hline
    Average & 0.39 & 0.43 & 0.57 & 0.69 & 0.60 & 0.62 & 0.72 & 0.75 & 0.55 & 0.64 & 0.73 & 0.77 & 0.56 & 0.66 & 0.73 & 0.77 & 0.43 & 0.53 & 0.68 & 0.75 & 0.42 & 0.51 & 0.60 & 0.71 & 0.67 & 0.65 & 0.74 & 0.78 \\
    \hline
    \end{tabular}%
%    \end{adjustbox}
  \label{tab:perf}%
\end{table}%

\end{landscape}

\begin{figure}[t]
     \centering
     \begin{subfigure}[b]{0.3\textwidth}
         \centering
         \includegraphics[width=\textwidth]{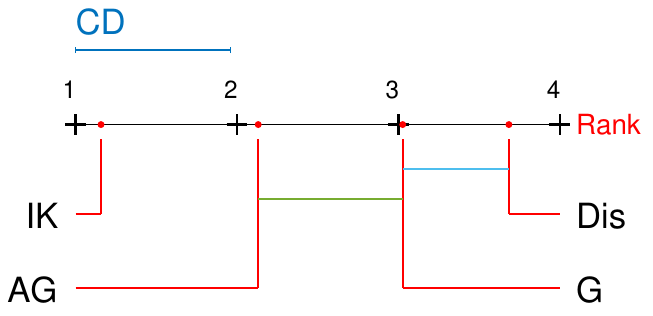}
         \caption{Single-linkage AHC}
         \label{fig:single}
     \end{subfigure}
    \begin{subfigure}[b]{0.3\textwidth}
       \centering
       \includegraphics[width=\textwidth]{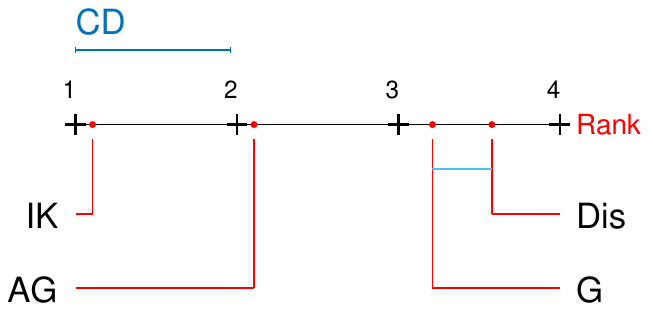}
      \caption{Complete-linkage AHC}
       \label{fig:complete}
    \end{subfigure}
     \begin{subfigure}[b]{0.3\textwidth}
         \centering
         \includegraphics[width=\textwidth]{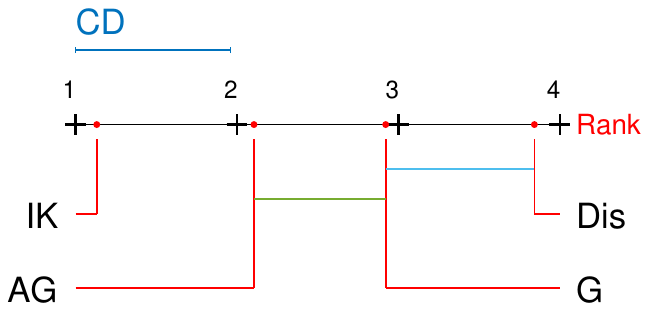}
         \caption{Average-linkage AHC}
         \label{fig:average}
     \end{subfigure}
     \begin{subfigure}[b]{0.3\textwidth}
         \centering
         \includegraphics[width=\textwidth]{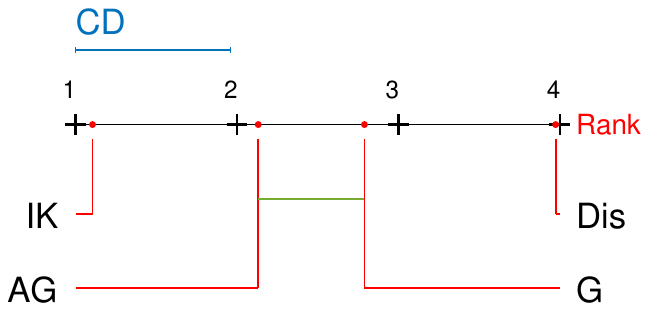}
         \caption{Weighted-linkage AHC}
         \label{fig:weighted}
     \end{subfigure}
     \begin{subfigure}[b]{0.3\textwidth}
         \centering
         \includegraphics[width=\textwidth]{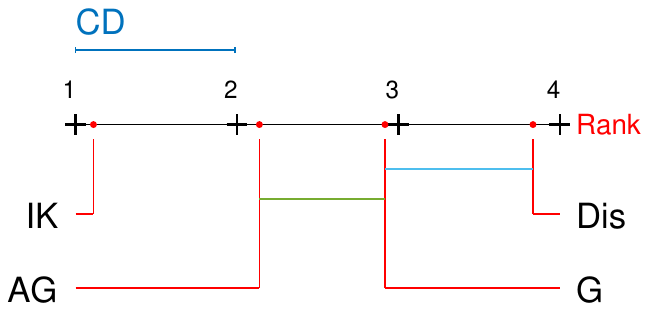}
         \caption{HDBSCAN}
         \label{fig:hdbscan}
     \end{subfigure}
     \begin{subfigure}[b]{0.3\textwidth}
         \centering
         \includegraphics[width=\textwidth]{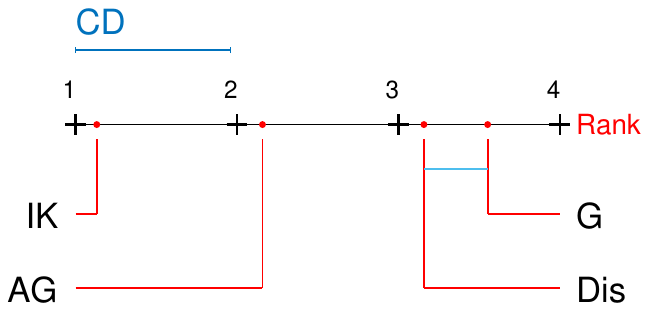}
         \caption{GDL}
         \label{fig:gdl}
     \end{subfigure}
     \begin{subfigure}[b]{0.3\textwidth}
         \centering
         \includegraphics[width=\textwidth]{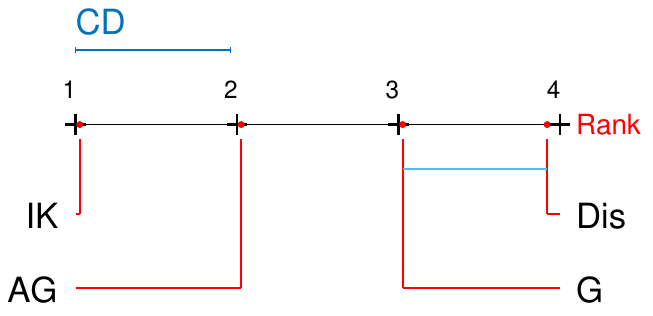}
         \caption{PHA}
         \label{fig:pha}
     \end{subfigure}
        \caption{Critical difference (CD) diagram of the post-hoc Nemenyi test ($\alpha=0.10$) for $F1$ scores. Two measures are not significantly different if there is a line linking them.}
        \label{fig:Nemenyitest}
\end{figure}

We also conduct a Friedman test with the post-hoc Nemenyi test to examine whether the difference in $F1$ scores of any two measures is significant.
As shown in Figure~\ref{fig:Nemenyitest},
IK is significantly better than all other measures for every algorithm.
This result provides further evidence of the superiority of IK wrt the clustering performance reported in the previous study where IK improves DBSCAN clustering results on datasets with varied densities \citep{qin2019nearest}.  

 \subsection{Computational complexity comparison}

 Table \ref{tab:complexity} compares the time and space complexities of different clustering algorithms and similarity measures.\footnote{Here we use the distance matrix as the input for each algorithm to save running time. When data points are available, the space complexity of all algorithms can be reduced to $\mathcal{O}(n)$. } Basically, all four measures have similar computational complexities. The runtime comparison on four datasets is shown in Table \ref{tab:runtime}. Note that the runtime of IK is slightly higher than the other two kernels because it is an ensemble method. However,
 this is not an issue because the Voronoi diagram implementation of IK is amenable to GPU acceleration \citep{qin2019nearest}.

\begin{table}[!htb]
  \centering
%   \footnotesize
  \caption{Time and space complexities of AHC algorithms and distance/kernel functions. $t$ and $\psi$ in IK are the ensemble size and subsample size, respectively. %  %For berAHC and HDBSCAN, the dissimilarity matrix $M$ is available as an input to the algorithm.
}
    \begin{tabular}{|c|c|c|}
    \hline
    Algorithm & Time complexity & Space complexity\\
    \hline
    Single-linkage AHC & $\mathcal{O}\left(n^{2}\right)$  & $\mathcal{O}\left(n^{2}\right)$ \\
    Complete-linkage AHC  & $\mathcal{O}\left(n^{2}\log n\right)$  & $\mathcal{O}\left(n^{2}\right)$ \\
    GDL & $\mathcal{O}\left(n^{3}\right)$ &  $\mathcal{O}\left(n^{2}\right)$\\
    PHA & $\mathcal{O}\left(n^{2}\right)$ & $\mathcal{O}\left(n^{2}\right)$ \\
  %  CLR & & \\
     Average  & $\mathcal{O}\left(n^{2}\log n\right)$  & $\mathcal{O}\left(n^{2}\right)$ \\
    Weighted  & $\mathcal{O}\left(n^{2}\log n\right)$  & $\mathcal{O}\left(n^{2}\right)$ \\
    HDBSCAN & $\mathcal{O}\left( n^{2}\right)$ & $\mathcal{O}\left( n^{2}\right)$\\
\hdashline
Distance or GK  & $\mathcal{O}\left(n^{2}\right)$ & $\mathcal{O}\left(n^{2}\right)$ \\
%GK &  $\mathcal{O}\left( n^{2}\right)$ & $\mathcal{O}\left(n^{2}\right)$ \\
AGK &    $\mathcal{O}\left(n^{2}\right)$ & $\mathcal{O}\left(n^{2}\right)$ \\
IK &  $\mathcal{O}\left( tn\psi+n^{2}\right)$ &  $\mathcal{O}\left(t \psi+ n^{2}\right)$  \\
    \hline
    \end{tabular}%
  \label{tab:complexity}%
\end{table}%

\begin{table}[!htb]
  \centering
  \caption{Execution time (in CPU seconds) on a machine with an i7-7820X 3.60GHz processor and 32GB RAM} % \textcolor{red}{Show AHC and HDBSCAN. AHC shows on one linkage function is sufficient.}}
    \begin{tabular}{|c|rrrr|rrrr|}
    \hline
    \multirow{2}[4]{*}{Dataset} & \multicolumn{4}{c|}{Single-linkage AHC} & \multicolumn{4}{c|}{HDBSCAN} \\
\cline{2-9}      & \multicolumn{1}{c}{Dis} & \multicolumn{1}{c}{GK} & \multicolumn{1}{c}{AGK} & \multicolumn{1}{c|}{IK} & \multicolumn{1}{c}{Dis} & \multicolumn{1}{c}{GK} & \multicolumn{1}{c}{AGK} & \multicolumn{1}{c|}{IK} \\ \hline
    WDBC  & .00 & .01 & .01 & .08 & .03 & .04 & .04 & .07 \\
    Banknote & .02 & .05 & .06 & .30 & .13 & .15 & .20 & .32 \\
    Segment  & .06 & .13 & .20 & .93  & .44 & .48 & .61 & .98 \\
    Spam   & .28 & .59 & .86 & 3.57 & 1.76 & 1.86 & 2.45 & 3.87 \\
    \hline
    \end{tabular}%
  \label{tab:runtime}%
\end{table}%

\newpage

\section{Conclusions}~\label{sec:conclusion}

We formally establish the condition with which a linkage function must comply before it would allow an AHC to successfully separate clusters in a dataset. We also formally define a concept called {\em entanglement} in a dendrogram to explain the severity of linking across different clusters during the merging process in the AHC. Two indicators, i.e., the number of entanglements and the average entanglement level, are shown to be highly correlated to an objective measure of goodness of dendrogram called dendrogram purity.

These formal definitions have allowed us to analyse an often overlooked/ignored bias in T-AHC: existing T-AHC to have a bias towards linking points in dense cluster first, before linking points in the sparse cluster.

As we contend that the root cause of this bias is due to the distance/similarity used being data-independent, the use of a well-defined \emph{data-dependent kernel called Isolation Kernel} has been shown to reduce this bias significantly.

While the analysis was conducted with respect to T-AHC only,
we propose to use Isolation Kernel to replace distance in existing distance-based AHC algorithms as a generic approach to improve their dendrograms. This approach differs from existing approaches which focus on a tailored-made linkage function for a specific algorithm. We show that the proposed approach works for four existing clustering algorithms without the need to modify their linkage functions or algorithms, except the replacement of distance with Isolation Kernel.

Our empirical evaluation verifies that Isolation Kernel is a better measure than distance and two existing popular kernels, Gaussian  Kernel  and  adaptive  Gaussian  Kernel, on four AHC algorithms, i.e., T-AHC, HDBSCAN \citep{campello2013density}, GDL~\citep{zhang2012graph} and PHA~\citep{lu2013pha}.
%Based on the dendrogram visualisation and clustering result, Isolation Kernel also has a better ability to equalise the cluster densities in kernel space than other two kernel methods.
%This is a direct result of Isolation  Kernel's better ability to equalise the cluster densities in kernel space. \textcolor{red}{Have you provided the evidence to support this claim?}

%In the future, we will explore the ability of kernel methods to improve the hierarchical clustering on large and complex datasets.

%based AHC produces the best clustering results on 19 real datasets compared with distance-based and two kernel-based hierarchical clustering algorithms.

%\textcolor{red}{Include more references with some AAAI papers.}

%\textcolor{red}{Reference [15]: The title needs to be corrected}

\clearpage

%
% The next two lines define the bibliography style to be used, and the bibliography file.
\bibliographystyle{elsarticle-num}
\bibliography{reference}

\begin{thebibliography}{10}
\expandafter\ifx\csname url\endcsname\relax
  \def\url#1{\texttt{#1}}\fi
\expandafter\ifx\csname urlprefix\endcsname\relax\def\urlprefix{URL }\fi
\expandafter\ifx\csname href\endcsname\relax
  \def\href#1#2{#2} \def\path#1{#1}\fi

\bibitem{jainDataClusteringReview1999}
A.~K. Jain, M.~N. Murty, P.~J. Flynn, Data clustering: A review, ACM computing
  surveys (CSUR) 31~(3) (1999) 264--323.

\bibitem{gilpin2013formalizing}
S.~Gilpin, S.~Nijssen, I.~Davidson, Formalizing hierarchical clustering as
  integer linear programming, in: Twenty-Seventh AAAI Conference on Artificial
  Intelligence, 2013.

\bibitem{cohen2019hierarchical}
V.~Cohen-Addad, V.~Kanade, F.~Mallmann-Trenn, C.~Mathieu, Hierarchical
  clustering: Objective functions and algorithms, Journal of the ACM (JACM)
  66~(4) (2019) 26.

\bibitem{rajaramanMiningMassiveDatasets2011}
A.~Rajaraman, J.~D. Ullman, Mining of Massive Datasets, {Cambridge University
  Press}, 2011.

\bibitem{diezpalacioNovelBrainPartition2015}
I.~Diez, P.~Bonifazi, I.~Escudero, B.~Mateos, M.~A. Mu{\~n}oz, S.~Stramaglia,
  J.~M. Cortes, A novel brain partition highlights the modular skeleton shared
  by structure and function, Scientific reports 5 (2015) 10532.

\bibitem{Malik2010}
H.~H. Malik, J.~R. Kender, D.~Fradkin, F.~Moerchen,
  \href{https://doi.org/10.1007/s10618-010-0172-z}{Hierarchical document
  clustering using local patterns}, Data Mining and Knowledge Discovery 21~(1)
  (2010) 153--185.
\newblock \href {http://dx.doi.org/10.1007/s10618-010-0172-z}
  {\path{doi:10.1007/s10618-010-0172-z}}.
\newline\urlprefix\url{https://doi.org/10.1007/s10618-010-0172-z}

\bibitem{Zhao2005}
Y.~Zhao, G.~Karypis, U.~Fayyad,
  \href{https://doi.org/10.1007/s10618-005-0361-3}{Hierarchical clustering
  algorithms for document datasets}, Data Mining and Knowledge Discovery 10~(2)
  (2005) 141--168.
\newblock \href {http://dx.doi.org/10.1007/s10618-005-0361-3}
  {\path{doi:10.1007/s10618-005-0361-3}}.
\newline\urlprefix\url{https://doi.org/10.1007/s10618-005-0361-3}

\bibitem{tumminelloCorrelationHierarchiesNetworks2010}
M.~Tumminello, F.~Lillo, R.~N. Mantegna, Correlation, hierarchies, and networks
  in financial markets, Journal of Economic Behavior \& Organization 75~(1)
  (2010) 40--58.

\bibitem{wu2009towards}
J.~Wu, H.~Xiong, J.~Chen, Towards understanding hierarchical clustering: A data
  distribution perspective, Neurocomputing 72~(10-12) (2009) 2319--2330.

\bibitem{heller2005bayesian}
K.~A. Heller, Z.~Ghahramani, Bayesian hierarchical clustering, in: Proceedings
  of the 22nd international conference on Machine learning, 2005, pp. 297--304.

\bibitem{Dasgupta2016CostFunction}
S.~Dasgupta, A cost function for similarity-based hierarchical clustering, in:
  Proceedings of the Forty-eighth Annual ACM Symposium on Theory of Computing,
  STOC '16, ACM, New York, NY, USA, 2016, pp. 118--127.

\bibitem{kingStepwiseClusteringProcedures1967}
B.~King, Step-wise clustering procedures, Journal of the American Statistical
  Association 62~(317) (1967) 86--101.

\bibitem{sokal1958u}
R.~Sokal, C.~Michener, A statistical method for evaluating systematic
  relationships, University of Kansas science bulletin (University of Kansas,
  1958).

\bibitem{zhang2012graph}
W.~Zhang, X.~Wang, D.~Zhao, X.~Tang, Graph degree linkage: Agglomerative
  clustering on a directed graph, in: European Conference on Computer Vision,
  Springer, 2012, pp. 428--441.

\bibitem{zhang2013agglomerative}
W.~Zhang, D.~Zhao, X.~Wang, Agglomerative clustering via maximum incremental
  path integral, Pattern Recognition 46~(11) (2013) 3056--3065.

\bibitem{ankerst1999optics}
M.~Ankerst, M.~M. Breunig, H.-P. Kriegel, J.~Sander,
  \href{http://doi.acm.org/10.1145/304182.304187}{Optics: Ordering points to
  identify the clustering structure}, in: Proceedings of the 1999 ACM SIGMOD
  International Conference on Management of Data, SIGMOD '99, ACM, New York,
  NY, USA, 1999, pp. 49--60.
\newblock \href {http://dx.doi.org/10.1145/304182.304187}
  {\path{doi:10.1145/304182.304187}}.
\newline\urlprefix\url{http://doi.acm.org/10.1145/304182.304187}

\bibitem{ertoz2003finding}
L.~Ert{\"o}z, M.~Steinbach, V.~Kumar, Finding clusters of different sizes,
  shapes, and densities in noisy, high dimensional data., in: SDM, 2003, pp.
  47--58.

\bibitem{campello2013density}
R.~J. G.~B. Campello, D.~Moulavi, A.~Zimek, J.~Sander,
  \href{http://doi.acm.org/10.1145/2733381}{Hierarchical density estimates for
  data clustering, visualization, and outlier detection}, ACM Trans. Knowl.
  Discov. Data 10~(1) (2015) 5:1--5:51.
\newblock \href {http://dx.doi.org/10.1145/2733381}
  {\path{doi:10.1145/2733381}}.
\newline\urlprefix\url{http://doi.acm.org/10.1145/2733381}

\bibitem{zhu2016density}
Y.~Zhu, K.~M. Ting, M.~J. Carman, Density-ratio based clustering for
  discovering clusters with varying densities, Pattern Recognition 60 (2016)
  983--997.

\bibitem{ting2019lowest}
K.~M. Ting, Y.~Zhu, M.~Carman, Y.~Zhu, T.~Washio, Z.-H. Zhou, Lowest
  probability mass neighbour algorithms: relaxing the metric constraint in
  distance-based neighbourhood algorithms, Machine Learning 108~(2) (2019)
  331--376.

\bibitem{qin2019nearest}
X.~Qin, K.~M. Ting, Y.~Zhu, V.~Lee, Nearest-neighbour-induced isolation
  similarity and its impact on density-based clustering, in: Proceedings of the
  33rd AAAI Conference on AI (AAAI 2019), AAAI Press, 2019.

\bibitem{klemela2009smoothing}
J.~S. Klemel{\"a}, Smoothing of multivariate data: density estimation and
  visualization, Vol. 737, John Wiley \& Sons, 2009.

\bibitem{huang2015density}
H.~Huang, S.~Yoo, D.~Yu, H.~Qin, Density-aware clustering based on aggregated
  heat kernel and its transformation, ACM Transactions on Knowledge Discovery
  from Data (TKDD) 9~(4) (2015) 1--35.

\bibitem{8166757}
D.~Marin, M.~Tang, I.~B. Ayed, Y.~Boykov, Kernel clustering: density biases and
  solutions, IEEE Transactions on Pattern Analysis and Machine Intelligence
  41~(1) (2019) 136--147.

\bibitem{zelnik2005self}
L.~Zelnik-Manor, P.~Perona, Self-tuning spectral clustering, in: Advances in
  neural information processing systems, 2005, pp. 1601--1608.

\bibitem{ting2018IsolationKernel}
K.~M. Ting, Y.~Zhu, Z.-H. Zhou, Isolation kernel and its effect on {SVM}, in:
  Proceedings of the 24th ACM SIGKDD International Conference on Knowledge
  Discovery and Data Mining, ACM, 2018, pp. 2329--2337.

\bibitem{lu2013pha}
Y.~Lu, Y.~Wan, Pha: A fast potential-based hierarchical agglomerative
  clustering method, Pattern Recognition 46~(5) (2013) 1227--1239.

\bibitem{ackerman2016characterization}
M.~Ackerman, S.~Ben-David, A characterization of linkage-based hierarchical
  clustering, The Journal of Machine Learning Research 17~(1) (2016)
  8182--8198.

\bibitem{pmlr19}
N.~Yadav, A.~Kobren, N.~Monath, A.~Mccallum, Supervised hierarchical clustering
  with exponential linkage, in: K.~Chaudhuri, R.~Salakhutdinov (Eds.),
  Proceedings of the 36th International Conference on Machine Learning, Vol.~97
  of Proceedings of Machine Learning Research, PMLR, Long Beach, California,
  USA, 2019, pp. 6973--6983.

\bibitem{murtagh1983survey}
F.~Murtagh, A survey of recent advances in hierarchical clustering algorithms,
  The Computer Journal 26~(4) (1983) 354--359.

\bibitem{krishnamurthy2012efficient}
A.~Krishnamurthy, S.~Balakrishnan, M.~Xu, A.~Singh, Efficient active algorithms
  for hierarchical clustering, arXiv preprint arXiv:1206.4672.

\bibitem{aggarwal2013data}
C.~C. Aggarwal, C.~K. Reddy, Data Clustering: Algorithms and Applications,
  Chapman and Hall/CRC Press, 2013.

\bibitem{shuming2002potential}
S.~Shuming, Y.~Guangwen, W.~Dingxing, Z.~Weimin, Potential-based hierarchical
  clustering, in: Object recognition supported by user interaction for service
  robots, Vol.~4, IEEE, 2002, pp. 272--275.

\bibitem{karypis1999chameleon}
G.~Karypis, E.-H.~S. Han, V.~Kumar, Chameleon: Hierarchical clustering using
  dynamic modeling, Computer~(8) (1999) 68--75.

\bibitem{zhao2009cyclizing}
D.~Zhao, X.~Tang, Cyclizing clusters via zeta function of a graph, in: Advances
  in Neural Information Processing Systems, 2009, pp. 1953--1960.

\bibitem{scholkopf1998nonlinear}
B.~Sch{\"o}lkopf, A.~Smola, K.-R. M{\"u}ller, Nonlinear component analysis as a
  kernel eigenvalue problem, Neural computation 10~(5) (1998) 1299--1319.

\bibitem{shawe2004kernel}
J.~Shawe-Taylor, N.~Cristianini, et~al., Kernel methods for pattern analysis,
  Cambridge university press, 2004.

\bibitem{hinneburg2007denclue}
A.~Hinneburg, H.-H. Gabriel, Denclue 2.0: Fast clustering based on kernel
  density estimation, in: International symposium on intelligent data analysis,
  Springer, 2007, pp. 70--80.

\bibitem{dhillon2004kernel}
I.~S. Dhillon, Y.~Guan, B.~Kulis, Kernel k-means: spectral clustering and
  normalized cuts, in: Proceedings of the tenth ACM SIGKDD international
  conference on Knowledge discovery and data mining, ACM, 2004, pp. 551--556.

\bibitem{kang2018unified}
Z.~Kang, C.~Peng, Q.~Cheng, Z.~Xu, Unified spectral clustering with optimal
  graph, in: Thirty-Second AAAI Conference on Artificial Intelligence, 2018.

\bibitem{xu2003document}
W.~Xu, X.~Liu, Y.~Gong, Document clustering based on non-negative matrix
  factorization, in: Proceedings of the 26th annual international ACM SIGIR
  conference on Research and development in informaion retrieval, 2003, pp.
  267--273.

\bibitem{macdonald2000kernel}
D.~MacDonald, C.~Fyfe, The kernel self-organising map, in: KES'2000. Fourth
  International Conference on Knowledge-Based Intelligent Engineering Systems
  and Allied Technologies. Proceedings (Cat. No. 00TH8516), Vol.~1, IEEE, 2000,
  pp. 317--320.

\bibitem{qin2004kernel}
A.~K. Qin, P.~N. Suganthan, Kernel neural gas algorithms with application to
  cluster analysis, in: Proceedings of the 17th International Conference on
  Pattern Recognition, 2004. ICPR 2004., Vol.~4, IEEE, 2004, pp. 617--620.

\bibitem{ester1996density}
M.~Ester, H.-P. Kriegel, J.~Sander, X.~Xu, A density-based algorithm for
  discovering clusters a density-based algorithm for discovering clusters in
  large spatial databases with noise, in: Proceedings of the Second
  International Conference on Knowledge Discovery and Data Mining, KDD'96, AAAI
  Press, 1996, pp. 226--231.

\bibitem{scholkopf2002learning}
B.~Sch{\"o}lkopf, A.~J. Smola, F.~Bach, et~al., Learning with kernels: support
  vector machines, regularization, optimization, and beyond, MIT press, 2002.

\bibitem{maaten2008visualizing}
L.~v.~d. Maaten, G.~Hinton, Visualizing data using t-sne, Journal of machine
  learning research 9~(Nov) (2008) 2579--2605.

\bibitem{aurenhammer1991voronoi}
F.~Aurenhammer, Voronoi diagrams---{A} survey of a fundamental geometric data
  structure, ACM Computing Surveys 23~(3) (1991) 345--405.

\bibitem{kobren2017hierarchical}
A.~Kobren, N.~Monath, A.~Krishnamurthy, A.~McCallum, A hierarchical algorithm
  for extreme clustering, in: Proceedings of the 23rd ACM SIGKDD International
  Conference on Knowledge Discovery and Data Mining, 2017, pp. 255--264.

\bibitem{monath2019scalable}
N.~Monath, A.~Kobren, A.~Krishnamurthy, M.~R. Glass, A.~McCallum, Scalable
  hierarchical clustering with tree grafting, in: Proceedings of the 25th ACM
  SIGKDD International Conference on Knowledge Discovery \& Data Mining, 2019,
  pp. 1438--1448.

\bibitem{borg2012applied}
I.~Borg, P.~J. Groenen, P.~Mair, Applied multidimensional scaling, Springer
  Science \& Business Media, 2012.

\bibitem{szeliski2010computer}
R.~Szeliski, Computer vision: algorithms and applications, Springer Science \&
  Business Media, 2010.

\bibitem{mcinnes2017accelerated}
L.~McInnes, J.~Healy, Accelerated hierarchical density based clustering, in:
  2017 IEEE International Conference on Data Mining Workshops (ICDMW), IEEE,
  2017, pp. 33--42.

\bibitem{Dua:2019}
D.~Dua, C.~Graff, \href{http://archive.ics.uci.edu/ml}{{UCI} machine learning
  repository} (2017).
\newline\urlprefix\url{http://archive.ics.uci.edu/ml}

\bibitem{demvsar2006statistical}
J.~Dem{\v{s}}ar, Statistical comparisons of classifiers over multiple data
  sets, The Journal of Machine Learning Research 7 (2006) 1--30.

\bibitem{Rijsbergen1979}
C.~J.~V. Rijsbergen, Information Retrieval, 2nd Edition, Butterworth-Heinemann,
  Newton, MA, USA, 1979.

\bibitem{larsen1999fast}
B.~Larsen, C.~Aone, Fast and effective text mining using linear-time document
  clustering, in: Proceedings of the fifth ACM SIGKDD international conference
  on Knowledge discovery and data mining, 1999, pp. 16--22.

\bibitem{aliguliyev2009performance}
R.~M. Aliguliyev, Performance evaluation of density-based clustering methods,
  Information Sciences 179~(20) (2009) 3583--3602.

\bibitem{kuhn1955hungarian}
H.~W. Kuhn, The hungarian method for the assignment problem, Naval Research
  Logistics Quarterly 2~(1-2) (1955) 83--97.

\bibitem{manning2010introduction}
C.~Manning, P.~Raghavan, H.~Sch{\"u}tze, Introduction to information retrieval,
  Natural Language Engineering 16~(1) (2010) 100--103.

\bibitem{vinh2010information}
N.~X. Vinh, J.~Epps, J.~Bailey, Information theoretic measures for clusterings
  comparison: Variants, properties, normalization and correction for chance,
  Journal of Machine Learning Research 11~(Oct) (2010) 2837--2854.

\bibitem{jain1988algorithms}
A.~K. Jain, R.~C. Dubes, Algorithms for clustering data, Englewood Cliffs:
  Prentice Hall, 1988.

\end{thebibliography}

\clearpage

%\onecolumn

%\section{Supplementary Material}
\appendix

\section{Interpretation of HDBSCAN as an AHC algorithm}
\label{appendA}
HDBSCAN~\citep{campello2013density} can be interpreted as a new kind of AHC algorithm which relies on a density-based linkage function, i.e., an AHC with  single-linkage linkage function and a particular dissimilarity measure. It uses the single-linkage function based on reachability-distance to merge two subclusters, motivated by the density-based clustering algorithm DBSCAN \citep{ester1996density}. The reachability-distance is defined as
\begin{equation*}
    d_{kReach}(x,y)=\max\{dist(x,y), dist_k(x), dist_k(y)\}
    \label{reach}
\end{equation*}
\noindent where $dist_k(x)$ is the distance between $x$ and $x$'s $k$-1-th nearest neighbour.

%Two points are density-reachable when they are  in each other's $\epsilon$-neighbourhood and the number of points %\textcolor{red}{replace `point' with `point' throughout}
%in both their neighbourhood regions is higher than a threshold $\beta$.

%Let $\mathcal{N}_\epsilon(x)$ be the $\epsilon$-neighbourhood of $x$, i.e.,
%\begin{equation}
%\mathcal{N}_\epsilon(x)=\lbrace y \in D ~|~ dist(x,y) \leqslant \epsilon \rbrace
%\label{neighbourhhod}
%\end{equation}

% If $x\in \mathcal{N}_\epsilon(y) \nonumber  \wedge  y\in \mathcal{N}_\epsilon(x) \wedge (\mathcal{N}_\epsilon(x)\geq \beta) \wedge  (\mathcal{N}_\epsilon(y)\geq \beta)$, then \textcolor{red}{This appears to be a reachability-distance. What is density-reachability?}
% \begin{eqnarray}
% ReachDis_{\epsilon}^{\beta}(x,y)= dist(x,y);
% \end{eqnarray}
% otherwise, $ReachDis_{\epsilon}^{\beta}(x,y)=\infty$.

The linkage function of HDBSCAN is defined as:

\begin{equation*}
    \check{h}(C_i,C_i)=  \min_{x\in C_i, y\in C_j} d_{kReach}(x,y)
\end{equation*}

%\textcolor{red}{This function is asymmetry. Should it be symmetry? Also what is the final distance? Is the following a correct definition:\\
%\begin{eqnarray}
%ReachDis_{\epsilon}^{\beta}(x,y)= \argmin [dist(x,y)\ |\ x\in \mathcal{N}_\epsilon(y) \nonumber  \wedge \\ y\in \mathcal{N}_\epsilon(x) \wedge (\mathcal{N}_\epsilon(x)\geq \beta) \wedge  (\mathcal{N}_\epsilon(y)\geq \beta)]
%\end{eqnarray}
%Also, what happen when $x \not\in \mathcal{N}_\epsilon(y)$?}

%\textcolor{red}{Replace $\beta$ with a different symbol e.g., $\beta$}

%Note that $ReachDis_{\epsilon}^{\beta}(x,y)=\infty$ if $x \not\in \mathcal{N}_\epsilon(y)$ or $y \not\in \mathcal{N}_\epsilon(x)$.

Similar to T-AHC,
%HDBSCAN first sets the $\beta$, and then
HDBSCAN gradually increases the reachability-distance as the height on the dendrogram to merge two most similar subclusters (based on the above linkage function) iteratively. This can be interpreted as: the larger the reachability-distance, the large number of points are linked~\footnote{In practice, HDBSCAN uses a more efficient way to compute the dendrogram based on the splitting on the Minimum Spanning Tree \citep{jain1988algorithms}.}.

Unlike T-AHC,
%rather than setting a global threshold on $\epsilon$ to extract clusters,
HDBSCAN uses a dynamic programming method to set different thresholds on  the dendrogram to extract optimal clusters with varied densities. Furthermore, clusters with the number of points less than a user-specified $c$ will be ignored and treated as noise, after the dendrogram building process.
%In order to extract ``significant" clusters only, noisy points are ignored when building the dendrogram, another parameter called minimum cluster size $c$ is used in HDBSCAN~\citep{campello2013density} A cluster which number of samples are smaller than the $c$ will be left as noise.
Since HDBSCAN is a density-based clustering algorithm, it can detect arbitrarily shaped clusters and identify noise in the dataset.

%HDBSCAN is slightly different from Definition \ref{def1}, since it does not rely on a single threshold to extract clusters.

To kernelise HDBSCAN, the kernel-based linkage function is defined as:
\begin{equation*}
     \check{\hslash}(C_i,C_j)= \max_{x\in C_i y\in C_j}  \min\{K(x,y), K_k(x), K_k(y)\}
\end{equation*}
\noindent where  $K_k(x)$ is the similarity between $x$ and $x$'s $k$-1-th most similar neighbour.

\section{Kernelised  PHA}~\label{sec:appendB}
The PHA~\citep{lu2013pha} converts the distance
between data points into potential values to measure the similarity between clusters.
Suppose that there is a data set with $n$ data points, denoted as $X = \{x_1, x_2, \dots x_n\}$.
The distance between two data points $x_i$ and $x_j$ is denoted as $dist(x_i,x_j)$.
The potential value of point $x_i$ received from point $x_j$ is calculated by
\begin{equation*}
\Phi_{x_{i}, x_{j}}=\left\{\begin{array}{ll}
{-\frac{1}{dist\left(x_{i}, x_{j}\right)}} & {\text { if } dist\left(x_{i}, x_{j}\right) \geq \lambda} \\
{-\frac{1}{\lambda}} & {\text { if } dist\left(x_{i}, x_{j}\right)<\lambda}
\end{array}\right.
\end{equation*}
where the parameter $\lambda$ is used to avoid the singularity problem
when $dist(x_i, x_j)$ is too small.

The total potential value of a
data point $x_a$ is defined as the sum of the potential values it
has received from all the other data points
\begin{equation*}
\Phi_{x_{a}}=\sum_{i=1, i \neq a}^{n} \Phi_{x_{a}, x_i}
\end{equation*}

% According to the potential values and the distances between
% data points, an EWT is constructed by linking data points to
% its closest point with a higher potential value than it.
The linkage function of PHA is defined as
\begin{equation*}
h(C_1 , C_2) = dist(s_1,s_2)
\end{equation*}
where $s_1$ and $s_2$ are from these two clusters respectively and be determined by:
\begin{equation*}
\begin{aligned}
&\text { If } C_{2} \leq C_{1},\left\{\begin{array}{l}
{s 1=\underset{k}{\operatorname{argmin}}\left(\Phi_{k} | k \in C_{1}\right)} \\
{s 2=\underset{k}{\operatorname{argmin}}\left(dist(k, s1) |\left(k \in C_{2}\right) \operatorname{AND}\left(\Phi_{k} \leq \Phi_{s 1}\right)\right)}
\end{array}\right.\\
&\text { If } C_{1} \leq C_{2},\left\{\begin{array}{l}
{s 2=\underset{k}{\operatorname{argmin}}\left(\Phi_{k} | k \in C_{2}\right)} \\
{s 1=\underset{k}{\operatorname{argmin}}\left(dist(k, s1) |\left(k \in C_{1}\right) \operatorname{AND}\left(\Phi_{k} \leq \Phi_{s 2}\right)\right)}
\end{array}\right.
\end{aligned}
\end{equation*}
where $C_{i} \leq C_{j}$ means $\exists x \in C_{i}\left(\forall y \in C_{j}\left(\Phi_{i} \leq \Phi_{j}\right)\right)$.
% The hierarchy of the data set can be read off from the EWT by sequentially merging the linked pair with the closest distance.

To kernelise the PHA, the potential value of point $x_i$ received from point $x_j$ is calculated by
\begin{equation*}
\hat{\Phi}_{x_{i}, x_{j}}=\left\{\begin{array}{ll}
{-\frac{1}{1 - K\left(x_{i}, x_{j}\right)}} & {\text { if } K\left(x_{i}, x_{j}\right)\leq \lambda} \\
{-\frac{1}{\lambda}} & {\text { if } K\left(x_{i}, x_{j}\right)>\lambda}
\end{array}\right.
\end{equation*}

The total potential value of a data point $x_a$ is defined as
\begin{equation*}
\hat{\Phi}_{x_{a}}=\sum_{i=1, i \neq a}^{n} \hat{\Phi}_{x_{a}, x_i}
\end{equation*}

The kernel-based linkage function for PHA is defined as
$\hslash(C_1,C_2) = 1- K(s1,s2)$
where $s_1$ and $s_2$ are from these two clusters respectively and be determined by:
\begin{equation*}
\begin{aligned}
&\text { If } C_{2} \leq C_{1},\left\{\begin{array}{l}
{s 1=\underset{k}{\operatorname{argmin}}\left(\hat{\Phi}_{k} | k \in C_{1}\right)} \\
{s 2=\underset{k}{\operatorname{argmax}}\left(K(k, s1) |\left(k \in C_{2}\right) \operatorname{AND}\left(\hat{\Phi}_{k} \leq \hat{\Phi}_{s 1}\right)\right)}
\end{array}\right.\\
&\text { If } C_{1} \leq C_{2},\left\{\begin{array}{l}
{s 2=\underset{k}{\operatorname{argmin}}\left(\hat{\Phi}_{k} | k \in C_{2}\right)} \\
{s 1=\underset{k}{\operatorname{argmax}}\left(K(k, s1) |\left(k \in C_{1}\right) \operatorname{AND}\left(\hat{\Phi}_{k} \leq \hat{\Phi}_{s 2}\right)\right)}
\end{array}\right.
\end{aligned}
\end{equation*}

\section{Kernelised GDL}~\label{sec:appendC}
The graph degree linkage (GDL) algorithm~\citep{zhang2012graph} begins with a number of initial small clusters, and iteratively merge two clusters with the maximum  similarity.
Suppose that there is a data set with $n$ data points, denoted as $X = \{x_1, x_2, \dots x_n\}$.% A directed graph $G = (V,E)$ is build in $X$, where $V$ is the set of vertices corresponding $X$ and $E$ is the set of edges connecting vertices.
The similarities are computed using the product of the average indegree and average outdegree in a $KNN$ graph, in which the vertices is $X$ and the weights for edges are defined as:
\begin{equation}
w_{i j}=\left\{\begin{array}{ll}
{\exp \left(-\frac{d i s t(i, j)^{2}}{\sigma^{2}}\right),} & {\text { if } x_{j} \in \mathcal{N}_{i}^{K}} \\
{0,} & {\text { otherwise }}
\end{array}\right.
\end{equation}
where $dist(i,j)$ is the distance between $x_i$ and $x_j$, $\mathcal{N}_{i}^{K}$ is the set of K-nearest neighbours of $x_i$, and $\sigma^{2}=\frac{a}{n K}\left[\sum_{i=1}^{n} \sum_{x_{j} \in \mathcal{N}_{i}^{K}} dist(i, j)^{2}\right]$ . $K$ and $a$ are free
parameters to be set.

Given a vertex $i$, the average indegree from and the average outdegree
to a cluster $C$ is defined as $\operatorname{deg}_{i}^{-}(\mathcal{C})=\frac{1}{|\mathcal{C}|} \sum_{j \in \mathcal{C}} w_{j i}$ and $\operatorname{deg}_{i}^{+}(\mathcal{C})=\frac{1}{|\mathcal{C}|} \sum_{j \in \mathcal{C}} w_{i j}$, respecitvely, where $|C|$ is the cardinality of $C$.

The similarity between two clusters is defined as the product
of the average indegree and average outdegree.
\begin{equation}
\mathcal{A}_{\mathcal{C}_{b}, \mathcal{C}_{a}}=\sum_{i \in \mathcal{C}_{b}} \operatorname{deg}_{i}^{-}\left(\mathcal{C}_{a}\right) \operatorname{deg}_{i}^{+}\left(\mathcal{C}_{a}\right) + \sum_{i \in \mathcal{C}_{a}} \operatorname{deg}_{i}^{-}\left(\mathcal{C}_{b}\right) \operatorname{deg}_{i}^{+}\left(\mathcal{C}_{b}\right)
\end{equation}

To kernelise the GDL, simply replacing the distance with a kernel in building the K-NN graph.
The weights of K-NN graph are defined as
\begin{equation}
\hat{w}_{i j}=\left\{\begin{array}{ll}
{\exp \left(-\frac{(1-K(i, j))^{2}}{\sigma^{2}}\right),} & {\text { if } x_{j} \in \mathcal{N}_{i}^{K}} \\
{0,} & {\text { otherwise }}
\end{array}\right.
\end{equation}
where $\mathcal{N}_{i}^{K}$ is the set of K-nearest neighbours of $x_i$.

\end{document}